\documentclass[review]{elsarticle}

\graphicspath{ {./} }
\usepackage{hyperref}
\hypersetup{
    colorlinks=true,
    linkcolor=blue,          
    citecolor=blue,          
    urlcolor=blue            
}
\usepackage{float}
\usepackage{verbatim} 
\usepackage{apalike}
\restylefloat{figure}
\restylefloat{table}
\usepackage{mathtools, array, booktabs}
\usepackage{tabularx}
\usepackage{pdflscape}
\usepackage{afterpage}
\usepackage{lipsum}
\usepackage{amsmath}  
\usepackage{amssymb}
\usepackage{amsthm}
\usepackage{booktabs}
\usepackage{threeparttable}
\usepackage{multirow}
\usepackage{bm}
\usepackage{url}
\usepackage{subfigure}
\usepackage{algorithm}
\usepackage{algorithmicx}
\usepackage{algpseudocode}
\usepackage{hyperref}

\usepackage{multicol}
\usepackage{caption}
\usepackage{graphicx}
\usepackage{textcomp}
\usepackage{amsfonts}

\usepackage{makecell}
\usepackage{siunitx}
\usepackage{etoolbox}
\usepackage{pdfpages}
\usepackage{wrapfig}
\usepackage{todonotes}
\newtheorem{theorem}{Theorem}
\newtheorem{lemma}{Lemma}
\newtheorem{proposition}{Proposition}

\biboptions{authoryear}

\begin{document}\sloppy
\begin{frontmatter}


\title{Overlap-aware meta-learning attention to enhance hypergraph neural networks for node classification}

\author[1]{Murong Yang}
\ead{mryang@shu.edu.cn}

\author[2,3]{Shihui Ying}
\ead{shying@shu.edu.cn}

\author[4,5]{Yue Gao}
\ead{gaoyue@tsinghua.edu.cn}

\author[1,6]{Xin-Jian Xu\corref{cor1}}
\ead{xinjxu@shu.edu.cn}

\address[1]{Department of Mathematics, College of Sciences, Shanghai University, Shanghai 200444, China}
\address[2]{
Shanghai Institute of Applied Mathematics and Mechanics, Shanghai University, Shanghai 200072, China}
\address[3]{School of Mechanics and Engineering Science, Shanghai University, Shanghai 200072, China}
\address[4]{BNRist, KLISS,
School of Software, Tsinghua University, Beijing 100084, China}
\address[5]{THUIBCS, Tsinghua University, Beijing 100084, China}
\address[6]{Qian Weichang College, Shanghai University, Shanghai 200444, China }

\cortext[cor1]{Corresponding author}

\begin{abstract}
Hypergraph neural networks (HGNNs) have emerged as a powerful framework for modeling higher-order interactions, yet this theoretical promise has not fully materialized in practice.
On one hand, existing networks typically employ a single type of attention mechanism, focusing on either structural or feature similarities during message passing.
On the other hand, assuming that all nodes in current hypergraph models have the same level of overlap may lead to suboptimal generalization.
To overcome these limitations, we propose a novel framework, Overlap-aware Meta-learning Attention for Hypergraph Neural Networks (OMA-HGNN).
First, we introduce a hypergraph attention mechanism that integrates both structural and feature similarities. Specifically, we linearly combine their respective losses with weighted factors for the HGNN model.
Second, we partition nodes into different tasks based on their diverse overlap levels and develop a multi-task Meta-Weight-Net (MWN) to determine the corresponding weighted factors.
Third, we jointly train the internal MWN model with the losses from the external HGNN model  and train the external model with the weighted factors from the internal model.
To evaluate the effectiveness of OMA-HGNN, we conducted experiments on six real-world datasets and benchmarked its performance against nine  state-of-the-art methods for node classification. The results demonstrate that OMA-HGNN excels in learning superior node representations and outperforms these baselines.
\end{abstract}


\end{frontmatter}

\section{Introduction}
\label{introduction}

Hypergraphs extend the concept of  dyadic graphs by enabling connections among more than two nodes, thereby providing a more accurate representation of higher-order relationships among data points in real-world scenarios. Due to their ability to model complex interactions, hypergraphs have demonstrated superior performance in various data mining tasks \cite{gao2020hypergraph, young2021hypergraph}. This versatile structure has been successfully applied across diverse domains, including social networks \cite{zhu2018social}, computer vision \cite{yu2012adaptive}, and natural language processing \cite{ding2020}.
To effectively capture these higher-order relationships, the design of an appropriate hypergraph model is crucial for fully leveraging the potential of this structure \cite{7707352}. Among various methods, hypergraph neural networks (HGNNs) have emerged as one of the most prominent models \cite{Yang2025Recent}. Similar to graph neural networks (GNNs) \cite{2016Semi, 2017Graph, yu2024gnn}, HGNNs integrate both the data structure and features as inputs. This integration allows HGNNs to exploit rich structural and feature information, leading to significant advancements in downstream tasks such as node classification \cite{2019Hypergraph, yadati2019hypergcn} and link prediction \cite{2021Heterogeneous,10.1145/3543507.3583256,wang2023hyconve}.
In recent years, numerous HGNN models have been proposed \cite{yadati2019hypergcn, 2020HyperSAGE, 2021UniGNN}.
These models effectively leverage hypergraph structures through message passing mechanisms to generate informative node features.
During message passing, information flows bidirectionally: from neighboring nodes to neighboring hyperedges, and then from these hyperedges back to the central node \cite{ chien2021you}. 
This bidirectional information flow enables nodes to better capture contextual relationships, thereby enhancing their representation quality and task performance. Typically, this process is guided by supervised signals derived from labeled nodes.

The first deep learning model designed specifically for hypergraphs is HGNN \cite{2019Hypergraph}. By leveraging the hypergraph Laplacian, HGNN has achieved notable success in visual object recognition tasks. However, it treats all connections uniformly, failing to account for the varying importance of different neighbors. Following similar developments in graph attention 
hypergraph attention approaches based on feature similarity (FS) have been introduced to dynamically capture the attention coefficients of neighbors \cite{9342986, Session-based2021}. Typically, these coefficients are derived by applying FS measures between a node and its incident hyperedges \cite{2020Hyper, bai2021hypergraph}. Meanwhile, recent advances in graph neural networks also emphasize the role of structural similarity (SS) in representation learning. For example, random walk kernels \cite{NEURIPS2020_ba95d78a} and quantum-based graph convolutions \cite{9521820} have been proposed to encode multi-scale and quantum structural patterns, while aligned  kernel-integrated filters and vertex convolutions \cite{Feng_You_Wang_Tassiulas_2022, 9646437} capture hierarchical and subgraph-level information beyond the 1-WL expressivity. These studies highlight that SS is crucial for assessing connection importance in graph models. Motivated by such findings, hypergraph neural networks have also started to incorporate SS metrics, such as node degree or shared neighbors, as alternative measures for calculating attention coefficients \cite{Jiang2020Node, Zareie2020Similarity-based}.

On one hand, existing HGNNs mostly rely on a single type of attention mechanism based either on SS or FS.
Spectral hypergraph convolution methods utilize the hypergraph Laplacian matrix to define convolution operations on hypergraphs \cite{2019Hypergraph, yadati2019hypergcn}. These methods can be viewed as capturing  SS between a central node and its neighboring hyperedges, thereby enabling efficient information aggregation. For example, Huang et al. \cite{2021UniGNN} proposed UniGCN, an extension of GCN that incorporates SS to reflect the relational connections and hierarchical positions of nodes.
In contrast, spatial hypergraph convolution methods typically leverage FS between a central node and its neighbor hyperedges during message passing. For instance, Zhang et al. \cite{2020Hyper} proposed Hyper-SAGNN, which employs an attention mechanism based on the average distance between static and dynamic embedding pairs of nodes. And Bai et al. \cite{bai2021hypergraph} introduced hypergraph convolution and hypergraph attention, calculating attention scores between nodes and their associated hyperedges with a feature-based similarity function. Similarly, Huang et al. \cite{2021UniGNN} proposed UniGAT, an extension of GAT for hypergraphs, which utilizes FS between nodes and their neighboring hyperedges.
Despite these advancements, these methods predominantly focus on a single type of attention, which limits their ability to capture the full range of information. Recently, Saxena et al. \cite{10.1145/3637528.3672047} proposed DPHGNN, a dual-perspective hypergraph neural network that integrates spectral inductive biases with spatial message passing through topology-aware attention and dynamic feature fusion. Nevertheless, DPHGNN primarily applies a single attention mechanism within each layer, leaving the integration of both attention mechanisms based on SS and FS as an open question.

A natural approach to sample weighting is to use a weighted linear combination of sample losses, where the weight function maps each training loss to a corresponding sample weight. Monotonically increasing functions assign larger weights to high-loss samples, effectively handling class imbalance in methods such as Boosting, AdaBoost \cite{friedman2000additive}, hard negative mining \cite{malisiewicz2011ensemble}, and focal loss \cite{lin2017focal}. Conversely, monotonically decreasing functions prioritize low-loss samples to mitigate noisy labels \cite{wang2017robust}. However, these predefined weighting schemes often rely on manually tuned hyperparameters or costly cross-validation. To address this, Shu et al. \cite{shu2019meta} proposed MW-Net, a meta-learning framework that adaptively learns sample weights from losses. 
Despite their effectiveness, these adaptive weighting strategies have not yet been applied in hypergraph models.

On the other hand, characterizing the microscopic structure of hypergraphs is essential for understanding both static structures and dynamic processes in complex systems, offering key insights for higher-order interaction modeling and representation learning. Recent studies have incorporated the concept of overlap into graph or hypergraph representation models \cite{NEURIPS2021_71ddb91e,LU2023109818}, but most rely on simple metrics such as common neighbors or resource allocation, which fail to capture the structural complexity of real-world hypergraphs. Other efforts have introduced novel overlap measures to analyze phase transitions \cite{Malizia2025Hyperedge}, yet these metrics were primarily designed for modeling critical phenomena rather than general structural properties. More recently, a principled measure satisfying three key axioms (hyperedge count, node distinctness, and hyperedge size) has been proposed, which provides a more comprehensive and theoretically grounded quantification of hyperedge overlaps \cite{10.1145/3442381.3450010}. This principled formulation lays a solid foundation for its integration into hypergraph neural networks, where it could help identify important structural regions, reduce redundancy, and adaptively handle different hyperedge scales; however, this direction has not yet been explored.

To overcome the aforementioned challenges, we propose a novel method named Overlap-aware Meta-learning Attention for Hypergraph Neural Networks (OMA-HGNN) for node classification. This method is a multi-task meta-learning hypergraph attention framework based on SS and FS while explicitly considering node overlap differences.
The main contributions of this paper are summarized as follows:

\begin{itemize}
    \item We integrate two attention mechanisms based on SS and FS by using a weighted linear combination of their losses for each sample as the total loss of the external HGNN model. To the best of our knowledge, this is the first attempt to use a combination of SS- and FS-based attention losses, weighted via a meta-learning approach, 
        enabling the model to capture both structural and feature-based relationships effectively.

    \item  We partition nodes into diverse tasks based on their overlap levels and develop a Multi-Task MWN (MT-MWN) to determine weighted factors. This work is the first to model  overlap levels with a multi-task framework  into hypergraph neural networks, to explicitly capture structural heterogeneity and achieve superior generalization across diverse hypergraph scenarios.

    \item We jointly train the external HGNN model and the internal MT-MWN model in each iteration. Specifically,
        the losses from the external model guide the parameter updates of the internal model, which in turn determines the weighted factors used to update the external model's parameters. This collaborative optimization enhances the model's ability to  automatically integrate attention fusion with node overlap-awareness.
\end{itemize}

The remainder of this paper is organized as follows. Section \ref{s2} reviews the related works. Section \ref{s3} provides the necessary preliminaries. Section \ref{s4} presents the proposed model OMA-HGNN, along with its optimization procedure. Section \ref{s5} demonstrates the effectiveness of OMA-HGNN through experimental comparisons. Finally, Section \ref{s6} concludes the paper.

\section{Related works}\label{s2}


\subsection{HGNNs}

Hypergraphs generalize traditional graphs by allowing hyperedges to connect any number of nodes, making them highly effective for modeling higher-order relationships. This structure is particularly useful in various real-world scenarios. For instance, coauthorship networks can be represented as hypergraphs, where hyperedges connect papers authored by the same individual \cite{yadati2019hypergcn}. In text analysis, hyperedges link articles containing the same keyword \cite{ding2020}, and in computer vision, they connect data points with similar features \cite{yu2012adaptive}. The generated hypergraphs can be utilized to build and train hypergraph models.
With the growing prominence of deep learning, HGNNs have garnered significant attention in recent researches \cite{2019Hypergraph, yadati2019hypergcn}.
Feng et al. \cite{2019Hypergraph} introduced HGNN, proposing a hypergraph Laplacian-based convolution operation that is approximated using truncated Chebyshev polynomials. Bai et al. \cite{bai2021hypergraph} integrated hypergraph attention mechanisms into the hypergraph convolution process, enabling the model to focus on the most informative relationships within the data. Yadati et al. \cite{yadati2019hypergcn} introduced a generalized hypergraph Laplacian that incorporates weighted pairwise edges, which enhances the model's robustness to noise.
In contrast to the above spectral convolution methods, spatial hypergraph convolution approaches directly propagate information on hypergraphs without defining the Laplacian matrix. Spatial models such as HyperSAGE \cite{2020HyperSAGE}, UniGNN \cite{2021UniGNN}, AllSet \cite{chien2021you} and KHGNN \cite{11063418} adopt a node-hyperedge-node information propagation scheme, which iteratively learns data representations. This aggregation process enables more flexible and direct message passing between nodes and hyperedges, making spatial HGNNs particularly well-suited for tasks requiring fine-grained relationship modeling.

\subsection{Message passing in HGNNs}
Message passing is a fundamental mechanism in HGNNs for learning node representations. This process is explicitly implemented in spatial convolution models. The core principle is that nodes connected within the same hyperedge tend to share similarities and are likely to have the same label \cite{zhang2017re}.
For example, MPNN-R \cite{yadati2020neural} treats hyperedges as nodes, thereby enabling the direct application of traditional message passing-based GNNs to hypergraphs. Another notable approach is AllSet \cite{chien2021you}, which integrates Deep Sets \cite{zaheer2017deep} and Set Transformer \cite{lee2019set} models to perform multiset-based message passing. This integration enhances the flexibility of the message passing process.
Building on these ideas, Arya et al. \cite{2020HyperSAGE} introduced HyperSAGE, which employs a two-stage message passing procedure to enhance the efficiency and effectiveness of node classification in hypergraphs. More recently, Huang et al. \cite{2021UniGNN} proposed UniGNN, a unified framework that generalizes the message passing process across both graphs and hypergraphs.
Despite these advancements, existing message passing mechanisms either neglect attention or rely solely on a single type of attention, which restricts their ability to capture nuanced relationships among nodes.

\subsection{Meta-learning for sample weighting}
Recent advancements in meta-learning have significantly influenced the development of methods for optimizing sample weight determination automatically \cite{hospedales2021meta, 2021Learning}. These methods are particularly relevant in addressing challenges associated with training deep neural networks using weakly labeled data.
For instance,
Dehghani et al. \cite{dehghani2017fidelity} proposed a semi-supervised approach that integrates label quality into the learning process to boost training efficiency. In addition, Fan et al. \cite{fan2018learning} introduced an automated teacher-student framework in which the student model learns from a teacher model that dynamically determines the optimal loss function and hypothesis space. Building on this, Jiang et al. \cite{jiang2018mentornet} developed a Student-Teacher network designed to learn a data-driven curriculum and enhance the learning procedure.
And Ren et al. \cite{ren2018learning} adaptively assigned weights to training samples based on gradient directions, eliminating the need for manual hyperparameter tuning. Later, Shu et al. \cite{shu2019meta} introduced MW-Net that derives explicit weighting schemes, demonstrating significantly better performance than traditional predefined methods.
While these works have optimized sample weighting for general deep learning tasks through meta-learning, they fall short in  handling the complexities of hypergraph structures and the integration of multi-attention mechanisms.

\subsection{Adaptive weighting via multi-task meta-learning}
From a task-distribution perspective, meta-learning aims to develop task-agnostic algorithms by leveraging knowledge from a diverse set of training tasks. This approach enables models to generalize effectively to new tasks by learning underlying patterns and adaptability mechanisms
\cite{hospedales2021meta}.
For example, MW-Net \cite{shu2019meta} derives a weighting function to optimize sample importance. However, its applicability relies on the assumption that all training tasks follow a similar distribution \cite{finn2017model, NEURIPS2020_0a716fe8}. In practice, this assumption is often violated in complex scenarios where tasks may exhibit heterogeneous bias configurations.
To resolve this issue, CMW-Net \cite{Shu2023} is proposed to adapt the meta-learning framework by treating each training class as an independent learning task and imposing class-specific weighting schemes. This approach enhances generalization in heterogeneous settings by accommodating diverse task distributions.
Despite these advancements, existing methods primarily focus on tasks involving pairwise relationships and fail to explicitly account for the high-order relationships inherent in real-world data.

\section{Preliminaries}\label{s3}

\subsection{Hypergraphs}

Unlike a simple graph, a hyperedge in a hypergraph can connect multiple nodes simultaneously. Given a hypergraph \(\mathcal{G} = (\mathcal{V}, \mathcal{E})\), where \(\mathcal{V}\) is the set of nodes and \(\mathcal{E}\) is the set of hyperedges, the structure of the hypergraph is represented by an incidence matrix \(\mathbf{H}\) with dimension \(|\mathcal{V}| \times |\mathcal{E}|\):
\[
\mathbf{H}(v_i, e_j) =
\begin{cases}
1, & \text{if } v_i \in e_j, \\
0, & \text{if } v_i \notin e_j,
\end{cases}
\]
where \(v_i \in \mathcal{V}\) and \(e_j \in \mathcal{E}\).

The feature matrix of the hypergraph \(\mathcal{G}\) is denoted as \(\mathbf{X} \in \mathbb{R}^{|\mathcal{V}| \times d}\), where each row \(\mathbf{x}_i\) of \(\mathbf{X}\) represents the \(d\)-dimensional feature vector of node \(v_i\). Additionally, \(\mathbf{y}_i \in \{0,1\}^C\) denotes the one-hot encoded class label vector corresponding to node \(v_i\).

\subsection{Two attention mechanisms based on SS and FS}
\label{3.2}

In graph convolutional networks, the update rule for node \(v_i\) is given by:

\begin{equation}
\mathbf{x}_i^{(t+1)} = \sigma\left( \mathbf{x}_i^{(t)} \mathbf{w}^{(t)} + \sum_{v_j \in N(v_i)} s_{ij}^{(t)} \mathbf{x}_j^{(t)} \mathbf{w}^{(t)} \right),
\end{equation}
where \(N(v_i)\) denotes the set of neighboring nodes of \(v_i\).

On one hand, GCN \cite{2016Semi} employs an attention mechanism based on SS, which is calculated as
\begin{equation}
s_{ij}^{(t)} = \frac{1}{d_i d_j},
\end{equation}
where \(d_i\) and \(d_j\) are the degrees of nodes \(v_i\) and \(v_j\), respectively. SS assigns higher weights to nodes with lower degrees, as these nodes play a more central role in information propagation.

On the other hand, GAT \cite{2017Graph} adopts an attention mechanism based on FS, where the attention coefficient is computed by
\begin{equation}
s_{ij}^{(t)} = \frac{\exp\left(\sigma\left(\mathbf{a}^{(t)} \left[\mathbf{x}_i^{(t)} \mathbf{w}^{(t)} \parallel \mathbf{x}_j^{(t)} \mathbf{w}^{(t)}\right]\right)\right)}{\sum_{v_k \in N(v_i)} \exp\left(\sigma\left(\mathbf{a}^{(t)} \left[\mathbf{x}_i^{(t)} \mathbf{w}^{(t)} \parallel \mathbf{x}_k^{(t)} \mathbf{w}^{(t)}\right]\right)\right)},
\end{equation}
where \(\mathbf{x}_i^{(t)}\) denotes the feature vector of node \(v_i\) at the \(t\)-th layer, and \(\mathbf{a}^{(t)}\) is a learnable attention vector.

\subsection{Two-stage message passing on hypergraphs}
\label{3.3}

HyperSAGE \cite{2020HyperSAGE} is a HGNN model inspired by GraphSAGE \cite{NIPS2017_5dd9db5e}. It employs a two-stage message passing process: first from neighbor nodes to neighbor hyperedges, and then from neighbor hyperedges to the central node. However, the aggregation process in HyperSAGE requires recalculating distinct neighbors for each central node, which leads to inefficient parallelism.

To overcome this limitation, UniGNN \cite{2021UniGNN} proposes a unified framework for message passing on both graphs and hypergraphs:

\begin{subequations}
\begin{align}
\label{eq:v}
& \underbrace{\mathbf{x}_{e} = \phi_1\left(\{\mathbf{x}_j\}_{v_j \in {e}}\right)}_{\text{stage 1: node-level aggregation}},  \\
\label{eq:e}
& \underbrace{\mathbf{x}_i = \phi_2\left(\mathbf{x}_i,\{\mathbf{x}_{e}\}_{{e} \in \mathcal{E}_i}\right)}_{\text{stage 2: hyperedge-level aggregation}},
\end{align}
\end{subequations}
where \(\phi_1\) and \(\phi_2\) represent permutation-invariant functions for node-level and hyperedge-level aggregation, respectively. \(\mathcal{E}_i\) denotes the set of all hyperedges containing node \(v_i\). Equations~\eqref{eq:v} and~\eqref{eq:e} describe the two-stage message passing process: node-level aggregation followed by hyperedge-level aggregation. Figure~\ref{fig1} illustrates an example of this process.

\begin{figure}[t]
  \centering
  \includegraphics[width=0.95\linewidth]{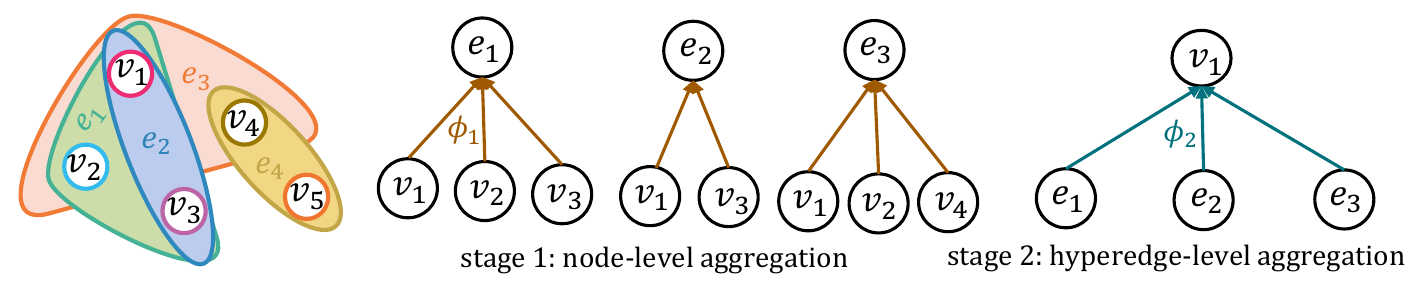}
  \caption{Two-stage message passing process in the hypergraph}
  \label{fig1}
\end{figure}

In UniGNN \cite{2021UniGNN}, self-loops are added to ensure that the central node's own information is preserved during aggregation, preventing it from being overly ``smoothed''. Two classic models are UniGCN and UniGAT. They are identical in the node-level aggregation step at stage 1 but differ in the hyperedge-level aggregation step at stage 2.

\subsection{Meta-learning based on MWN}

Deep neural networks often overfit when trained on biased or noisy datasets, as they treat all samples equally and can be misled by erroneous samples. To address this issue, meta-learning strategies are commonly adopted. In these approaches, the neural network acts as an external model, while an internal model modulates the influence of each sample on the training process. Specifically, the internal model takes the sample loss from the external model as input and generates corresponding sample weights. By assigning higher weights to important samples and lower weights to noisy ones, the internal model reduces the impact of erroneous samples while enhancing the influence of informative ones. This mechanism helps prevent overfitting and improves generalization.

MWN \cite{shu2019meta} is an example of such an internal model. It learns an adaptive sample weighting function using a small, unbiased meta-dataset rather than relying on manually preset weights. Implemented as a Multilayer Perceptron (MLP), MWN maps the sample training losses to corresponding sample weights. Each time the external model (with parameters \(\mathbf{w}\)) is trained and a loss is obtained, this loss is used as input to update the internal model's parameters (\(\theta\)). Based on the updated parameters of the internal model, new sample weights are generated and fed back into the external model for the next training iteration.

\section{The proposed method}\label{s4}

\subsection{Integrating attention mechanisms for HGNNs}
\label{s4.1}

As discussed in Sections~\ref{3.2} and~\ref{3.3}, when GCN and GAT are extended to hypergraphs, they become UniGCN and UniGAT, respectively. In this context, the hyperedge-level aggregation in stage 2 of hypergraphs can also be interpreted as an attention mechanism based on SS and FS. Unlike GCN and GAT, the attention coefficients shift from \(s_{ij}^{(t)}\) (for neighbor nodes in graphs) to \(s_{ie}^{(t)}\) (for neighbor hyperedges in hypergraphs).

After obtaining the hyperedge features \(\mathbf{x}_e\) through node-level aggregation in stage 1, the update rule for node \(v_i\) in stage 2 is given by:

\begin{equation}
\mathbf{x}_i^{(t+1)} = \sigma\left( \mathbf{x}_i^{(t)} \mathbf{w}^{(t)} + \sum_{e \in \mathcal{E}_i} s_{ie}^{(t)} \mathbf{x}_e^{(t)} \mathbf{w}^{(t)} \right),
\end{equation}
where \(s_{ie}^{(t)}\) is the attention coefficient between the central node \(v_i\) and its neighbor hyperedge \(e\). It can be computed using either SS or FS:

\begin{equation}
s_{ie}^{(t)} = \left\{
\begin{array}{ll}
\frac{1}{d_i d_e}, & \text{SS} \\
\frac{\exp\left(\sigma\left(\mathbf{a}^{(t)} \left[\mathbf{x}_i^{(t)} \mathbf{w}^{(t)} \parallel \mathbf{x}_e^{(t)} \mathbf{w}^{(t)}\right]\right)\right)}{\sum_{e \in \mathcal{E}(v_i)} \exp\left(\sigma\left(\mathbf{a}^{(t)} \left[\mathbf{x}_i^{(t)} \mathbf{w}^{(t)} \parallel \mathbf{x}_e^{(t)} \mathbf{w}^{(t)}\right]\right)\right)}, & \text{FS}
\end{array}
\right.
\end{equation}
Here, \(d_e\) denotes the average degree of nodes within hyperedge \(e\), and \(\mathcal{E}(v_i)\) represents the set of all hyperedges containing node \(v_i\).

\begin{figure*}[!ht]
\centering	
\subfigure[CA-Cora]{
\includegraphics[width=.288\textwidth]{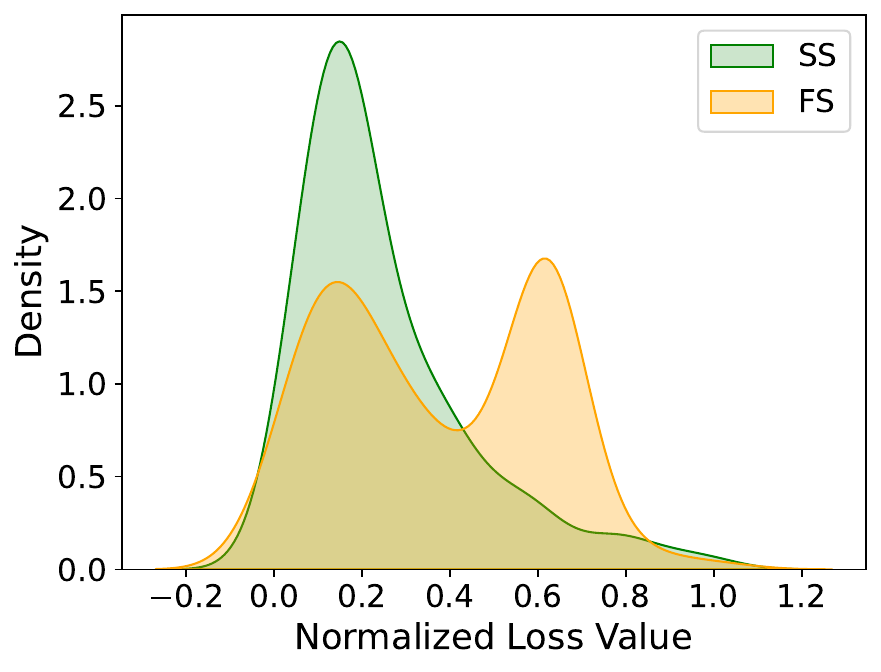}}
\subfigure[Citeseer]{
\includegraphics[width=.288\textwidth]{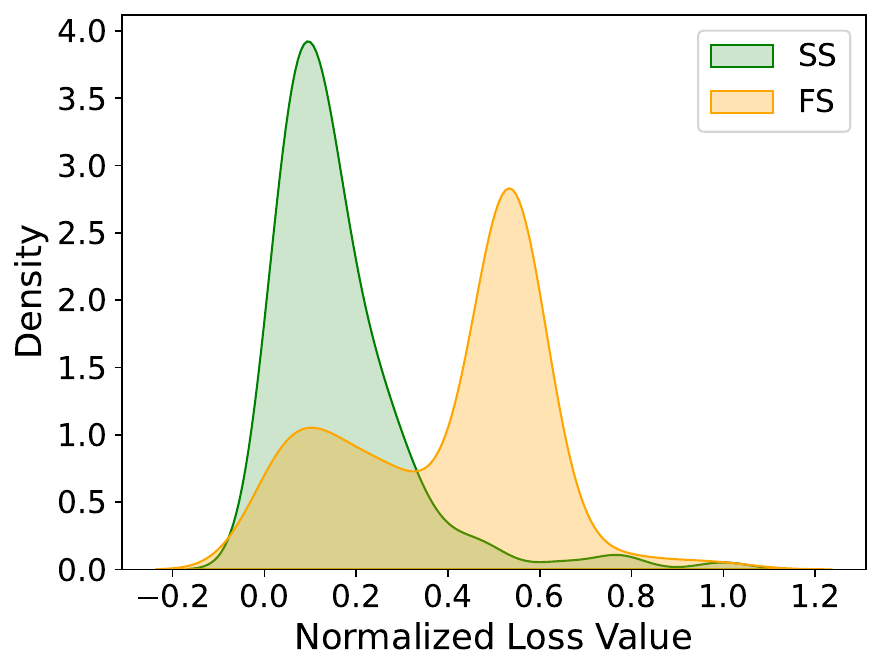}}
\subfigure[20news]{
\includegraphics[width=.28\textwidth]{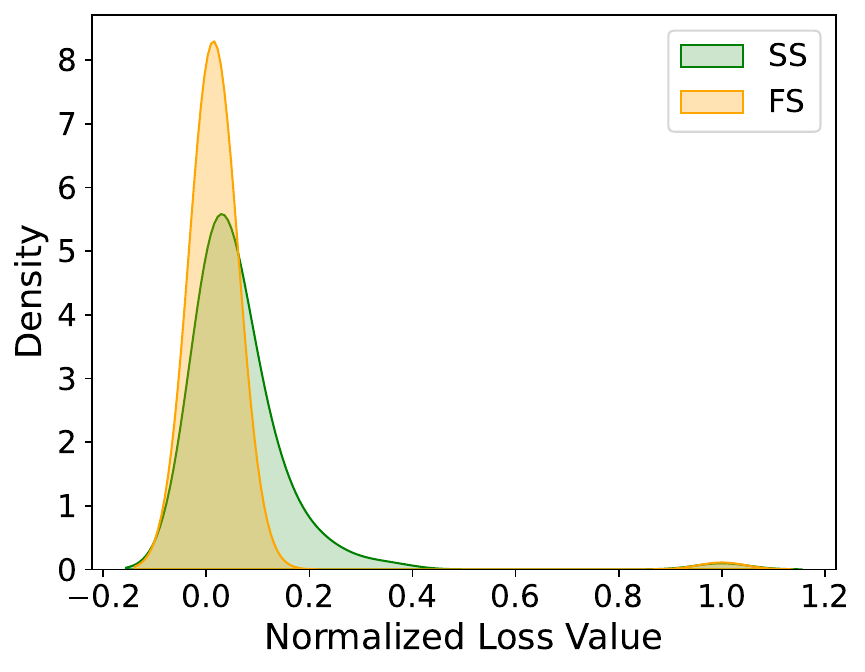}}

\subfigure[Reuters]{
\includegraphics[width=.288\textwidth]{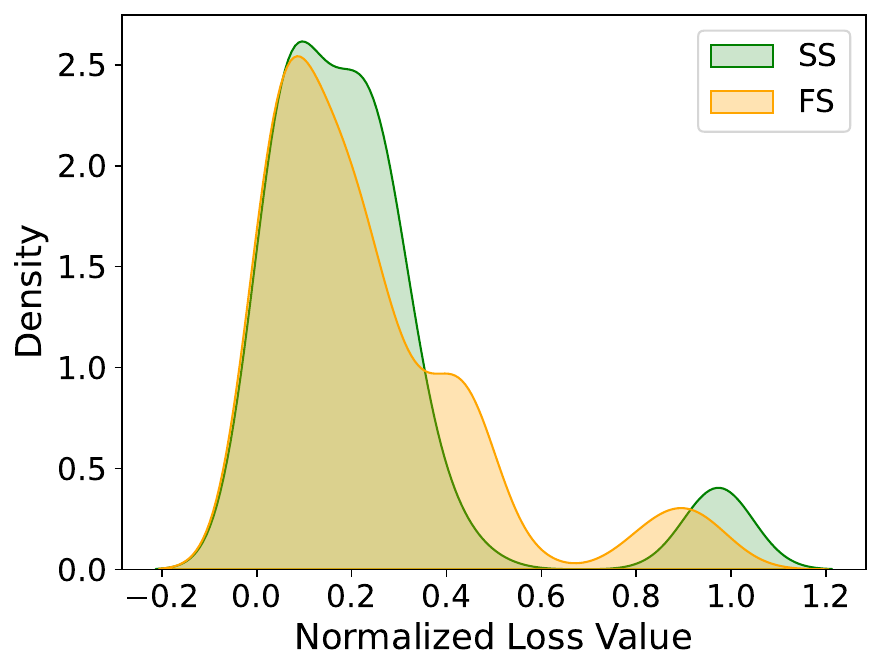}}
\subfigure[ModelNet]{
\includegraphics[width=.288\textwidth]{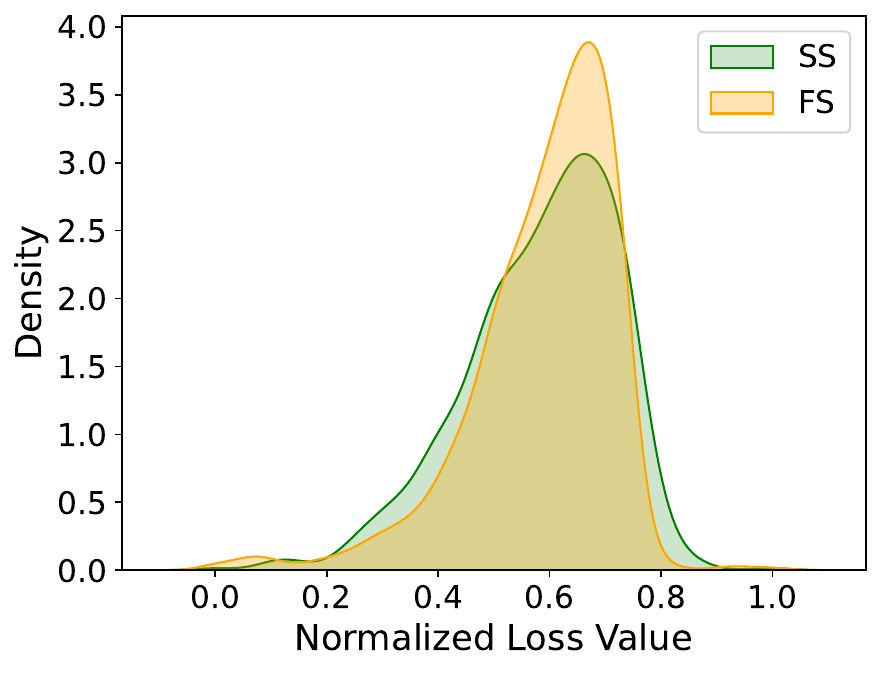}}
\subfigure[Mushroom]{
\includegraphics[width=.288\textwidth]{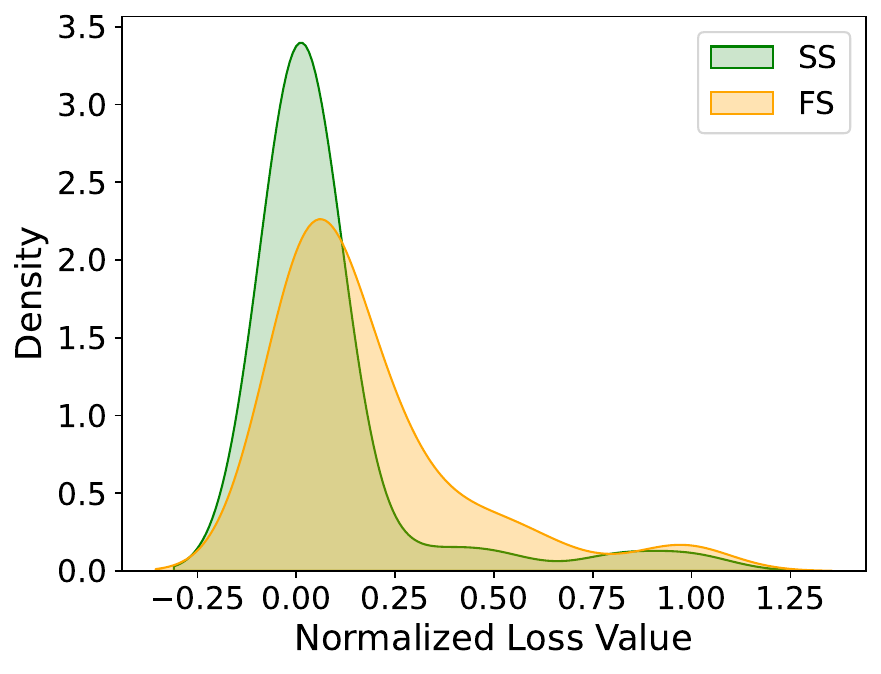}}
	\caption{Kernel density estimation curves of sample losses from two attention mechanisms.}
	\label{fig2}
\end{figure*}

To assess the performance differences between the two hypergraph attention mechanisms based on SS and FS, we plot the Kernel Density Estimation (KDE) curves of their corresponding sample losses, as shown in Figure~\ref{fig2}. 
From this figure, it is evident that there is a significant discrepancy in the sample losses between these two attention mechanisms for the same dataset. This suggests that each mechanism captures different aspects of the data: SS focuses on the structural properties, while FS emphasizes the feature-based relationships. Therefore, combining these two attention mechanisms is essential to fully leverage their complementary strengths. By integrating SS and FS, we aim to achieve a more comprehensive and robust representation of hypergraphs, thereby enhancing the overall learning performance.

Consider the classification task with training set \(\mathcal{D}^{\text{train}} = \left\{\mathbf{x}_i^{\text{train}}, \mathbf{y}_i^{\text{train}}\right\}_{i=1}^n\), where \(\mathbf{x}_i^{\text{train}}\) denotes the feature vector of node \(v_i\), \(\mathbf{y}_i^{\text{train}} \in \{0,1\}^C\) is the one-hot encoded label corresponding to \(v_i\), and \(n\) is the number of training samples. Let \(f(\cdot)\) denote the node classifier of the HGNN model with \(\mathbf{w}\) representing its model parameters. The optimal model parameters \(\mathbf{w}^*\) are obtained by minimizing the following objective function:

\begin{equation}
\mathbf{w}^* = \underset{\mathbf{w}}{\arg \min} \frac{1}{n} \sum_{i=1}^n \ell(f(\mathbf{x}_i^{\text{train}}, \mathbf{H}; \mathbf{w}), \mathbf{y}_i^{\text{train}}),
\end{equation}
where \(\ell(f(\mathbf{x}_i^{\text{train}}, \mathbf{H}; \mathbf{w}), \mathbf{y}_i^{\text{train}})\) denotes the training loss for sample \(v_i\). In this study, we adopt the commonly used cross-entropy (CE) loss \cite{10.1145/1102351.1102422}:
\[
\ell(f(\mathbf{x}_i^{\text{train}}, \mathbf{H}; \mathbf{w}), \mathbf{y}_i^{\text{train}}) = -{\mathbf{y}_i^{\text{train}}}^T \log(f(\mathbf{x}_i^{\text{train}}, \mathbf{H}; \mathbf{w})),
\]
where \(f(\mathbf{x}_i^{\text{train}}, \mathbf{H}; \mathbf{w})\) denotes the output embedding corresponding to \(v_i\).

To integrate the attention mechanisms based on SS and FS, a straightforward approach is to use a weighted linear combination of their losses as the total objective function for the external classifier model. Let \(\alpha_i \in [0,1]\) and \(\beta_i \in [0,1]\) denote the weights assigned to the SS-based and FS-based losses, respectively. The optimal model parameters \(\mathbf{w}^*\) can be obtained by minimizing the following total training loss:

\begin{equation}
\mathbf{w}^*(\theta) = \underset{\mathbf{w}}{\arg \min} \sum_{i=1}^n \alpha_i L_{i,\text{train}}^1(\mathbf{w}) + \beta_i L_{i,\text{train}}^2(\mathbf{w}),
\end{equation}
where \(L_{i,\text{train}}^1(\mathbf{w}) = \ell(f_{1}(\mathbf{x}_i^{\text{train}}, \mathbf{H}; \mathbf{w}), \mathbf{y}_i^{\text{train}})\) is the loss from the SS-based attention, and \(L_{i,\text{train}}^2(\mathbf{w}) = \ell(f_{2}(\mathbf{x}_i^{\text{train}}, \mathbf{H}; \mathbf{w}), \mathbf{y}_i^{\text{train}})\) is the loss from the FS-based attention.

However, setting \(\alpha_i = \beta_i\) is often suboptimal, as it assumes equal contributions from both attention mechanisms. This approach may lead to one mechanism being overemphasized, failing to adapt to heterogeneous data with attention bias. To address this, we introduce meta-learning to dynamically optimize the balance between the SS-based and FS-based losses.

\subsection{Overlap-aware meta-learning attention}

In this work, we adopt the same architecture as MWN \cite{shu2019meta} for the internal model that generates the sample weights. Specifically, the internal model is implemented as a Multilayer Perceptron (MLP), denoted by \( V\left(L_{i,\text{train}}^1(\mathbf{w}), L_{i,\text{train}}^2(\mathbf{w}); \theta\right) \), which takes the sample training losses as input and outputs sample weights. The network parameters of this internal model are denoted by \(\theta\), and the output is activated by the Sigmoid function.

While the external classifier model is trained on the training set \(\mathcal{D}^{\text{train}}\), the parameters \(\theta\) of the internal model are learned on a validation set \(\mathcal{D}^{\text{meta}}\). This separation allows us to capture the distinct characteristics of training and validation samples. Specifically, the differences between SS and FS are more pronounced in the training samples, whereas these differences are relatively smaller in the validation samples. The validation set \(\mathcal{D}^{\text{meta}} = \left\{\mathbf{x}_i^{\text{meta}}, \mathbf{y}_i^{\text{meta}}\right\}_{i=1}^m\) (where \(m \ll n\)) serves as a balanced subset with unbiased attention mechanisms. The optimal parameters \(\theta^*\) can be obtained by solving the following bi-level optimization problem:

\vspace{-10pt}
\begin{subequations} \label{eq3}
\begin{equation}
\label{eq3_1}
\begin{split}
&\,\,\textbf{\small{Internal Model:}}  \\[-4pt]
&\,\,\mathbf{\theta}^*=\underset{\theta}{\arg \min }  \frac{1}{m} \sum_{i=1}^m \left(L_{i,\text{meta}}^1(\mathbf{w}^*(\mathbf{\theta}))+L_{i,\text{meta}}^2(\mathbf{w}^*(\mathbf{\theta}))\right),
\end{split}
\end{equation}
\vspace{-5pt}
\begin{equation}
\label{eq3_2}
\begin{split}
&\textbf{\small{External Model:}} \\[-4pt]
&\text{s.t.} \quad \mathbf{w}^*(\theta) = \underset{\mathbf{w}}{\arg \min } \sum_{i=1}^n \alpha_i L_{i,\text{train}}^1(\mathbf{w}) + \beta_i L_{i,\text{train}}^2(\mathbf{w}) \\
& \quad\quad\,\, \alpha_i, \beta_i = V\left(L_{i,\text{train}}^1(\mathbf{w}),L_{i,\text{train}}^2(\mathbf{w}); \theta\right) \hphantom{\sum_{i=1}^n}
\end{split}
\end{equation}
\end{subequations}
\vspace{2pt}
where
$
L_{i,\text{meta}}^{1}(\mathbf{w}^*(\mathbf{\theta}))=\ell_{1}(f(\mathbf{x}_i^{\text{meta}},\mathbf{H}; \mathbf{w}^*(\mathbf{\theta})), \mathbf{y}_i^{\text{meta}})
$
and
$
L_{i,\text{meta}}^{2}(\mathbf{w}^*(\mathbf{\theta}))=\ell_{2}(f(\mathbf{x}_i^{\text{meta}},\mathbf{H}; \mathbf{w}^*(\mathbf{\theta})), \mathbf{y}_i^{\text{meta}})
$
are the validation losses from hypergraph attentions based on SS and FS for a validation sample \(v_i\), respectively.

As discussed in Section~\ref{s4.1}, these attention mechanisms operate at stage 2. Before this stage, hyperedge features are derived from node-level aggregation at stage 1, as described in Equation~\eqref{eq:v}. During this process, the overlap levels of neighboring nodes significantly influence the intensity of information propagation. For instance, higher overlap levels may lead to redundant information propagation, increasing the likelihood of over-smoothing. Conversely, lower overlap levels may enhance information discrimination but are more susceptible to noise. Traditional single-task learning frameworks train a unified set of parameters for all nodes, which limits the model's generalization ability when handling diverse overlap biases. To overcome this challenge, we partition nodes into distinct tasks based on their overlap levels and assign task-specific parameter spaces to each task.

To quantify the overlap level of a node, we employ a principled measure: the overlapness \( p_i \) of node \( v_i \)'s egonet \( E\{v_i\} \) (see Definition 2 in \cite{Lee2021}). This metric indicates the neighbor overlap level of node \( v_i \) and is defined as follows:

\begin{equation}
p_i := \frac{\sum_{e \in E\{v_i\}} |e|}{\left| \bigcup_{e \in E\{v_i\}} e \right|}.
\label{eq:p}
\end{equation}

Here, the numerator \(\sum_{e \in E\{v_i\}} |e|\) represents the sum of the sizes of all hyperedges within the egonet \( E\{v_i\} \) of node \( v_i \). The denominator \(\left| \bigcup_{e \in E\{v_i\}} e \right|\) denotes the size of the union of all nodes in the hyperedges of \( v_i \)'s egonet. Figure~\ref{fig03} provides an example of a hypergraph to illustrate Equation~(\ref{eq:p}).

\begin{figure}[t]
  \centering
  \includegraphics[width=0.25\linewidth]{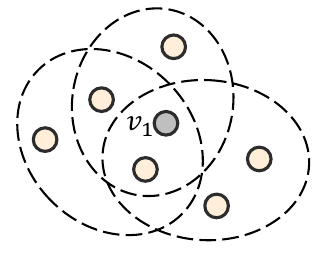}\\
  \caption{An example of a hypergraph,  where \(v_1\)'s   overlap level is $p_1=\frac{4+4+4}{7}=1.714$. }
  \label{fig03}
\end{figure}

This metric satisfies three fundamental axioms: it simultaneously accounts for the number of hyperedges, the number of distinct nodes, and the sizes of hyperedges. This comprehensive characterization distinguishes it from other commonly used metrics, such as Intersection, Union Inverse, Jaccard Index, Overlap Coefficient, and Density, none of which can concurrently satisfy all three axioms \cite{Lee2021}. Consequently, ratio \( p_i \) effectively characterizes the overlap level of \( v_i \). A higher \( p_i \) value indicates a more densely connected neighborhood structure, suggesting stronger relationships between \( v_i \) and its neighbors. In contrast, a lower \( p_i \) value implies a more dispersed neighborhood, which may indicate weaker relationships between \( v_i \) and its neighbors.

To identify the overlap levels of nodes, we introduce an overlap vector \( \mathbf{p} = [p_1, p_2, \dots, p_n]^T \) for all training nodes. By inputting \( \mathbf{p} \) into a parameterized partitioning function \(\mathcal{C}(\mathbf{p}; \Omega)\), nodes are categorized into \( K \) distinct levels, each assigned a task label \( c \in \{0, 1\}^K \). In this work, we apply the \( K \)-means clustering algorithm with \( K = 3 \) to cluster nodes into low, medium, and high overlap levels. Subsequently, we employ a multi-task learning strategy to assign task-specific weights based on varying overlap levels, thereby enhancing the model's generalization capability.

Specifically, we construct an MT-MWN:
\begin{equation}
\label{MWN}
\mathcal{V}(L_{i,\text{train}}^1(\mathbf{w}), L_{i,\text{train}}^2(\mathbf{w}), p_i; \Theta, \Omega),
\end{equation}
which adaptively learns sample weights. The shared layer extracts common features across all tasks to improve generalization, while the task-specific layer learns unique weights \(\{\alpha_i, \beta_i\}\) for each sample based on its task label \( c \). This design ensures that each overlap level (low, medium, high) has an independent weight generation module with unique parameters, as illustrated in Figure~\ref{fig3}. This strategy generates weights tailored to each overlap level, enhancing the model's flexibility and generalization.

\begin{figure}[tb]
  \centering
  \includegraphics[width=0.68\linewidth]{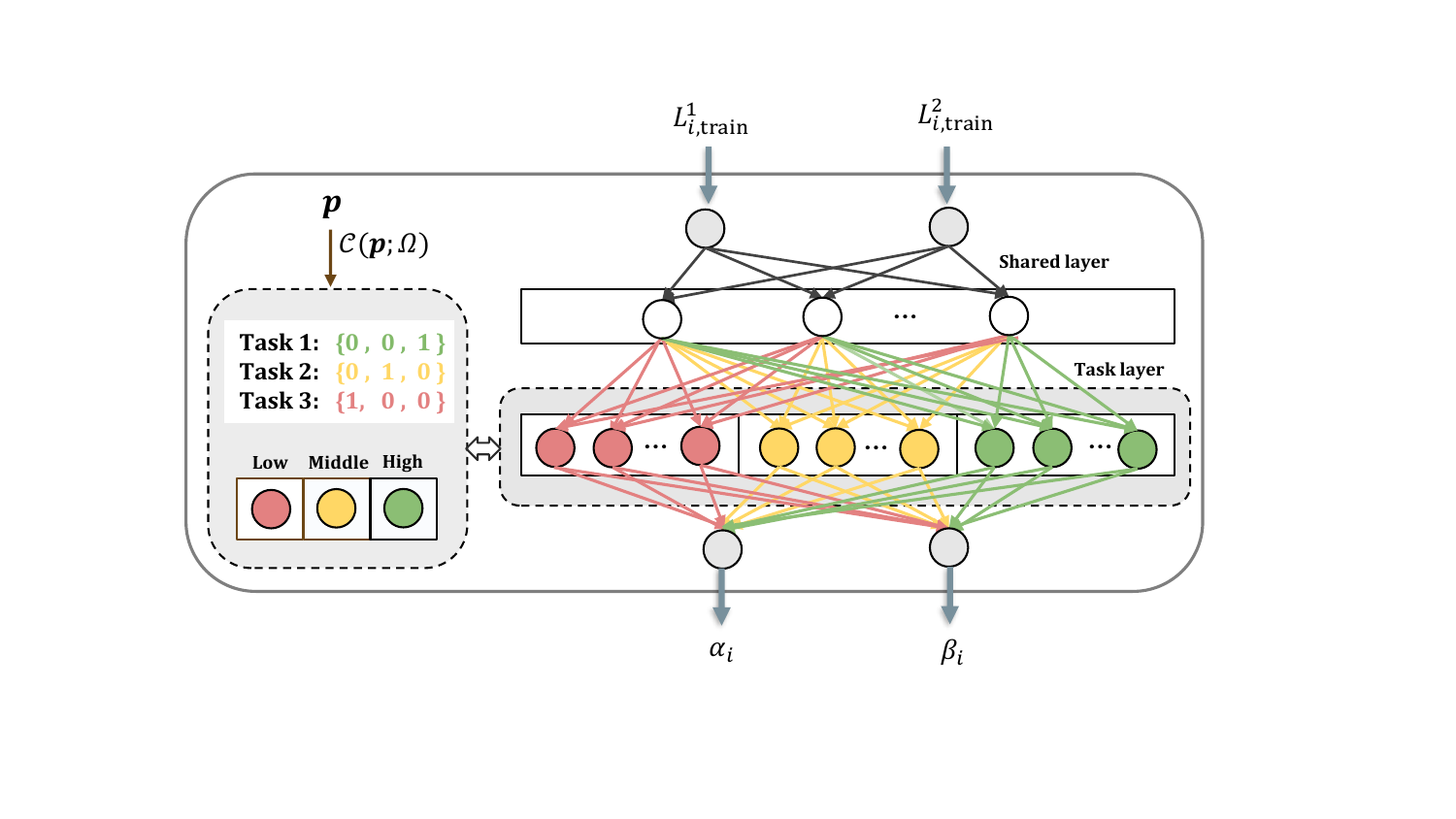}
  \caption{The architecture of MT-MWN, i.e., the internal model.}
  \label{fig3}
\end{figure}

During optimization, the MT-MWN  (\ref{MWN}) requires tuning of two parameter sets: \(\Theta\) and \(\Omega\). Here, \(\Theta\) includes all trainable parameters, encompassing both the shared layer and task-specific layer parameters. \(\Omega\), derived from \( K \)-means clustering, represents the centroids for each overlap level. Since initializing \(\Omega\) involves solving an integer programming problem, we adopt a two-stage approach: first, we apply \( K \)-means clustering to establish the initial centroids for overlap levels, denoted as \(\Omega^* = \{\mu_k\}_{k=1}^3\); then, we optimize \(\Theta\) to further enhance model performance. Given that \(\Omega^*\) is predetermined and \(p_i\) is provided, we directly denote the MT-MWN as
\begin{equation}
\mathcal{V}(L_{i,\text{train}}^1(\mathbf{w}), L_{i,\text{train}}^2(\mathbf{w}); \Theta).
\end{equation}

Now, the objective function of the model in Equation~(\ref{eq3}) can be reformulated as the following bi-level optimization problem:

\vspace{-14pt}
\begin{subequations} \label{eq4}
\begin{equation}
\label{eq4_1}
\begin{split}
&\,\,\,\,\,\quad\quad\,\quad\,\, \textbf{\small{Internal Model:}}  \\[-4pt]
&\,\,\,\,\,\quad\quad\,\quad\,\, \Theta^*  =\underset{\Theta, \Omega}{\arg \min }  \frac{1}{m} \sum_{i=1}^m L_{i,\text{meta}}^1\left(\mathbf{w}^*(\Theta)\right)+L_{i,\text{meta}}^2\left(\mathbf{w}^*(\Theta )\right)
\end{split}
\end{equation}
\vspace{-5pt}
\begin{equation}
\label{eq4_2}
\begin{split}
&\textbf{\small{External Model:}} \\[-4pt]
&\mathbf{w}^*(\Theta) =\underset{\mathbf{w}}{\arg \min } \sum_{i=1}^n \alpha_i L_{i,\text{train}}^1(\mathbf{w}) + \beta_i L_{i,\text{train}}^2(\mathbf{w}) \\
& \quad\quad \text { s.t. }\quad
  \alpha_i, \beta_i =\mathcal{V}\left(L_{i,\text{train}}^1(\mathbf{w}),L_{i,\text{train}}^2(\mathbf{w}); \Theta\right)
\end{split}
\end{equation}
\end{subequations}

\subsection{Optimization procedure}

Our objective is to solve the bi-level optimization problem defined by Eqs.~(\ref{eq4_1}) and~(\ref{eq4_2}) in order to obtain the optimal values for \(\Theta^*\) and \(\mathbf{w}^*\). It is worth noting that finding exact solutions for these equations requires solving for the optimal \(\mathbf{w}^*\) every time \(\Theta\) is updated.
To achieve this, we proceed with the following steps:
Step 1 computes an intermediate state of the external model parameters \(\hat{\mathbf{w}}^{(t)}(\Theta)\);
Step 2 updates the internal model parameters as \(\Theta^{(t+1)}\);
Step 3  updates the external model parameters as \(\mathbf{w}^{(t+1)}\).

\vspace{0.25cm}

\textbf{Step 1: intermediate state of external model parameters}

We first compute an intermediate state \(\hat{\mathbf{w}}^{(t)}(\Theta)\) of the external model parameters. At each iteration, we update the current parameters \(\mathbf{w}^{(t)}\) in the direction of the gradient of the objective function. We apply gradient descent on a mini-batch training set \(\mathcal{D}^{\text{train}} = \{\mathbf{x}_i^{\text{train}}, \mathbf{y}_i^{\text{train}}\}_{i=1}^n\), and the parameter is computed by:

\begin{equation}
\label{eq8}
\begin{aligned}
\hat{\mathbf{w}}^{(t)}(\Theta) &= \mathbf{w}^{(t)} - \lambda_{1} \sum_{i=1}^n \left(\alpha_i^{(t)} \nabla_{\mathbf{w}} L_{i,\text{train}}^1(\mathbf{w}^{(t)}) + \beta_i^{(t)} \nabla_{\mathbf{w}} L_{i,\text{train}}^2(\mathbf{w}^{(t)})\right) \Big|_{\mathbf{w}^{(t)}}.
\end{aligned}
\end{equation}

Here, \(\alpha_i^{(t)}\) and \(\beta_i^{(t)}\) are coefficients that depend on the loss functions \(L_{i,\text{train}}^1\) and \(L_{i,\text{train}}^2\) evaluated at the current internal model parameters \(\Theta^{(t)}\). \(\lambda_1\) is the learning rate for the external model, and \(\hat{\mathbf{w}}^{(t)}(\Theta)\) represents the intermediate external model parameters. This step provides an intermediate calculation to guide the subsequent updates to the internal model parameters \(\Theta\).

\vspace{0.25cm}

\textbf{Step 2: update of internal model parameters}

After computing the intermediate external model parameters \(\hat{\mathbf{w}}^{(t)}(\Theta)\) based on Eq.~(\ref{eq8}), we update the internal model parameters \(\Theta\) by applying gradient descent on the meta-objective defined by Eq.~(\ref{eq4_1}). This update is performed on the validation set \(\mathcal{D}^{\text{meta}} = \{\mathbf{x}_i^{\text{meta}}, \mathbf{y}_i^{\text{meta}}\}_{i=1}^m\), and the update rule is as follows:

\begin{equation}
\label{eq9}
\begin{aligned}
\Theta^{(t+1)} &= \Theta^{(t)} - \lambda_2 \frac{1}{m} \sum_{i=1}^m \left( \nabla_{\Theta} L_{i,\text{meta}}^1(\hat{\mathbf{w}}^{(t)}(\Theta)) + \nabla_{\Theta} L_{i,\text{meta}}^2(\hat{\mathbf{w}}^{(t)}(\Theta)) \right) \Big|_{\Theta^{(t)}}.
\end{aligned}
\end{equation}

Here, \(\lambda_2\) is the learning rate for the internal model, and the gradients are computed with respect to \(\Theta^{(t)}\), using the intermediate external model parameters \(\hat{\mathbf{w}}^{(t)}(\Theta)\) calculated in Step 1. This step ensures that the internal model learns in a manner consistent with the external model's state.

\vspace{0.25cm}

\textbf{Step 3: update of external model parameters}

Once the internal model parameters \(\Theta^{(t+1)}\) have been updated, we update the external model parameters \(\mathbf{w}\) by performing a gradient descent step. This step fine-tunes \(\mathbf{w}\) to the updated internal model. The update is performed using the updated \(\Theta^{(t+1)}\) calculated in Step 2, and the rule is as follows:

\begin{equation}
\label{eq10}
\begin{aligned}
\mathbf{w}^{(t+1)} &= \mathbf{w}^{(t)} - \lambda_{1} \sum_{i=1}^n \left(\alpha_i^{(t+1)} \nabla_{\mathbf{w}} L_{i,\text{train}}^1(\mathbf{w}^{(t)}) + \beta_i^{(t+1)} \nabla_{\mathbf{w}} L_{i,\text{train}}^2(\mathbf{w}^{(t)})\right) \Big|_{\mathbf{w}^{(t)}}.
\end{aligned}
\end{equation}

Here, \(\alpha_i^{(t+1)}\) and \(\beta_i^{(t+1)}\) are updated based on the new internal model parameters \(\Theta^{(t+1)}\). This step completes the update for the external model's parameters, ensuring they are aligned with the updated internal model.

\begin{figure*}[htbp]
  \centering
  \includegraphics[width=0.99\linewidth]{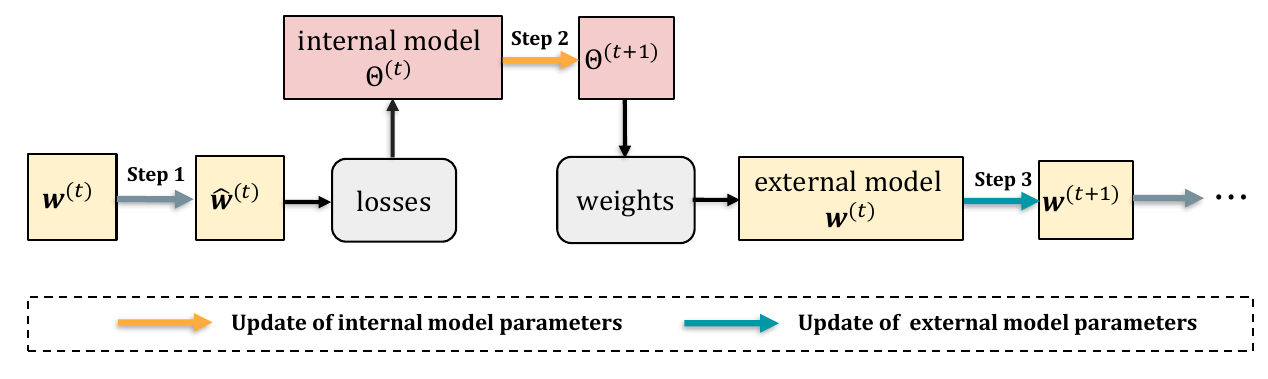}\\
  \caption{Schematic diagram of OMA-HGNN. The internal model is the  overlap-aware Meta-Weight Net, and the external model is the attention-fused HGNN classifier. }
  \label{fig4}
\end{figure*}

\begin{algorithm}
\caption{Iterative algorithm of OMA-HGNN}
	\label{alg1}
	\begin{algorithmic}[1]  
		\Require Training set $\mathcal{D}^{\text{train}}$, validation set
$\mathcal{D}^{\text{meta}}$, train size $n$, validation size $m$, max iterations $T$, incidence matrix  $\mathbf{H}$.

		\Ensure   External and internal model parameters $\mathbf{w}^{*}$, $\Theta^{*}$.
        \State Compute overlap $\mathbf{p}$ on $\mathcal{D}^{\text{train}}$ by Eq.(\ref{eq:p}).
        \State Perform $K$-means on $\mathbf{p}$ to obtain  $\Omega^*=\left\{\mu_k\right\}_{k=1}^K$.
        \State Initialize external and internal model parameters $\mathbf{w}^{(1)}$, $\Theta^{(1)}$.
        \For{each iteration $t\in [1,T]$}
        \State $\alpha_i^{(t)}, \beta_i^{(t)} =\mathcal{V}(L_{i,\text{train}}^1(\mathbf{w}^{(t)}),L_{i,\text{train}}^2(\mathbf{w}^{(t)}); \Theta^{(t)})$.
        \State  Calculate \(\hat{\mathbf{w}}^{(t)}(\Theta)\) by Eq.(\ref{eq8}) as an intermediate state.
        \State Update parameter $\Theta^{(t+1)}$ of internal model by Eq.(\ref{eq9}).
        \State $\alpha_i^{(t+1)}, \beta_i^{(t+1)} =\mathcal{V}(L_{i,\text{train}}^1(\mathbf{w}^{(t)}),L_{i,\text{train}}^2(\mathbf{w}^{(t)}); \Theta^{(t+1)})$.
        \State Update parameter $\mathbf{w}^{(t+1)}$ of external model by Eq.(\ref{eq10}).
        \EndFor
	\end{algorithmic}
\end{algorithm}

Figure~\ref{fig4} illustrates the schematic diagram of OMA-HGNN. The iterative algorithm of OMA-HGNN is summarized in Algorithm~\ref{alg1}.
It is evident that both the external model and internal model undergo gradual improvements in their parameters throughout the learning process. This iterative approach enables stable and consistent updates of the weights, ensuring that the models adapt effectively to the data.


\subsection{Convergence Analysis}

We perform a convergence analysis of Algorithm~\ref{alg1}. We theoretically prove that the proposed method converges to critical points of both the meta loss (\ref{eq4_1}) and the training loss (\ref{eq4_2}) under certain conditions. These results are presented in Theorems~\ref{th1} and~\ref{th2}, respectively. The detailed proofs are provided in the supplementary material (Appendix A).

We denote the overall meta loss as
\begin{equation}
\label{eq:ml}
\mathcal{L}^{\text{meta}}(\mathbf{w}^*(\Theta)) = \frac{1}{m} \sum_{i=1}^m \left( L_{i,\text{meta}}^1(\mathbf{w}^*(\Theta)) + L_{i,\text{meta}}^2(\mathbf{w}^*(\Theta)) \right),
\end{equation}
where \(\mathbf{w}^*\) is the parameters of the external model, and \(\Theta\) is the parameters of the internal model. Additionally, we denote the overall training loss as
\begin{equation}
\label{eq18}
\mathcal{L}^{\text{train}}(\mathbf{w}; \Theta) = \sum_{i=1}^n \left( \alpha_i L_{i,\text{train}}^1(\mathbf{w}) + \beta_i L_{i,\text{train}}^2(\mathbf{w}) \right),
\end{equation}
where \(\alpha_i, \beta_i = \mathcal{V}(L_{i,\text{train}}^1(\mathbf{w}), L_{i,\text{train}}^2(\mathbf{w}); \Theta)\).

\begin{theorem}\label{th1}
Suppose the loss functions \(\ell_1\) and \(\ell_2\) are Lipschitz smooth with constant \(M\), and \(\mathcal{V}(\cdot, \cdot; \Theta)\) is differentiable with a \(\delta\)-bounded gradient and twice differentiable with its Hessian bounded by \(\mathcal{B}\). Additionally, assume that the gradients of \(\ell_1\) and \(\ell_2\) are \(\rho\)-bounded with respect to the training/validation sets. Let the learning rates \(\lambda_1^{(t)}, \lambda_2^{(t)}\) for \(1 \leq t \leq T\) be monotonically decreasing sequences satisfying:
\(
\lambda_1^{(t)} = \min\left\{\frac{1}{M}, \frac{c_1}{\sqrt{T}}\right\}, \quad \lambda_2^{(t)} = \min\left\{\frac{1}{M}, \frac{c_2}{\sqrt{T}}\right\},
\)
for some \(c_1, c_2 > 0\) such that \(\frac{\sqrt{T}}{c_1} \geq M\) and \(\frac{\sqrt{T}}{c_2} \geq M\). Additionally, these sequences satisfy:
\(
\sum_{t=1}^{\infty} \lambda_1^{(t)} = \infty, \quad \sum_{t=1}^{\infty} (\lambda_1^{(t)})^2 < \infty,
\)
\(
\sum_{t=1}^{\infty} \lambda_2^{(t)} = \infty, \quad \sum_{t=1}^{\infty} (\lambda_2^{(t)})^2 < \infty.
\)
Under these conditions, the internal model can achieve
\[
\mathbb{E}\left[\left\|\nabla \mathcal{L}^{\text{meta}}\left(\hat{\mathbf{w}}^{(t)}(\Theta^{(t)})\right)\right\|_2^2\right] \leq \epsilon
\]
in \(\mathcal{O}(1/\epsilon^2)\) steps. More specifically,
\begin{equation}
\label{eq25}
\min_{0 \leq t \leq T} \mathbb{E}\left[\left\|\nabla \mathcal{L}^{\text{meta}}\left(\hat{\mathbf{w}}^{(t)}(\Theta^{(t)})\right)\right\|_2^2\right] \leq \mathcal{O}\left(\frac{C}{\sqrt{T}}\right),
\end{equation}
where \(C\) is a constant independent of the convergence process. (Proof: see Appendix A.1.)
\end{theorem}

\begin{theorem}\label{th2}
Under the conditions of Theorem~\ref{th1}, the external model can achieve
\[
\mathbb{E}\left[\left\|\nabla \mathcal{L}^{\text{train}}\left(\hat{\mathbf{w}}^{(t)}; \Theta^{(t)}\right)\right\|_2^2\right] \leq \epsilon
\]
in \(\mathcal{O}(1/\epsilon^2)\) steps. More specifically,
\begin{equation}
\label{eq33}
\min_{0 \leq t \leq T} \mathbb{E}\left[\left\|\nabla \mathcal{L}^{\text{train}}\left(\hat{\mathbf{w}}^{(t)}; \Theta^{(t)}\right)\right\|_2^2\right] \leq \mathcal{O}\left(\frac{C}{\sqrt{T}}\right),
\end{equation}
where \(C\) is a constant independent of the convergence process. (Proof: see Appendix A.2.)
\end{theorem}

\subsection{Time complexity analysis}

Let$|\mathcal{V}|$ be the number of nodes and $|\mathcal{E}|$ be the number of hyperedges. The number of nonzero entries in $\mathbf{H}$ is denoted by $\mathrm{nnz}(\mathbf{H})$. Equivalently, $\mathrm{nnz}(\mathbf{H}) = |\mathcal{V}|\, d_\mathcal{V} = |\mathcal{E}|\, d_\mathcal{E}$, where $d_\mathcal{V}$ denotes the average degree of nodes  and $d_\mathcal{E}$ denotes the average degree of hyperedges.

For the \emph{external} HGNN classifier, the costs are dominated by two sparse aggregations (node to hyperedge and hyperedge to node) and a feature transformation, resulting in $O(\mathrm{nnz}(\mathbf{H})\,d)+O(|\mathcal{V}| d^2)=O(|\mathcal{V}| d_\mathcal{V} d+|\mathcal{V}| d^2)$ per layer. The feature-similarity attention mechanism requires an additional $O(\mathrm{nnz}(\mathbf{H})\,d)$ of the same order, while the structural-similarity attention mechanism is based on degree or neighbor statistics that can be precomputed once in $O(\mathrm{nnz}(\mathbf{H}))$ and thus negligible during training. For the \emph{internal} MT-MWN, the forward weight generation introduces $O(|\mathcal{V}| h)$ per round, where $h$ is the hidden dimension of the meta-network, and the update on the meta set requires $O(m_{\text{meta}}h)$ with $m_{\text{meta}}\ll |\mathcal{V}|$, which is usually negligible compared with the external HGNN. One-time costs include computing node overlap levels $O(|\mathcal{V}| d_\mathcal{V})=O(|\mathcal{E}| d_\mathcal{E})$ and running $K$-means clustering $O(|\mathcal{V}| K)$ with small $K$. Therefore, the overall training complexity per layer is
\[
O\!\left(|\mathcal{V}| d_\mathcal{V} d+|\mathcal{V}| d^2+|\mathcal{V}| h\right),
\]
which reduces to $O(|\mathcal{V}| d_\mathcal{V} d+|\mathcal{V}| d^2)$ when $h$ is much smaller than $d$. 
Intuitively, the dominant cost comes from the sparse hypergraph aggregations and feature transformations, while the meta-weight module introduces only a lightweight linear overhead.

\section{Experiments}\label{s5}


\subsection{Datasets}

\begin{table}[htbp]
\caption{Dataset Statistics}
\centering
\scriptsize
\begin{tabular}{ccccccc}
\toprule
Dataset & Domains & Nodes & Hyperedges &  Features  & Classes & Training Size \\
\midrule
{CA-Cora} & Network & 2708 & 1072 & 1433 & 7 & 140 \\
{Citeseer} & Network & 3312 & 1079 & 3703 & 6 & 138 \\
{20news} & Text & 16242 & 100 & 100 & 4 & 80 \\
{Reuters} & Text & 10000 & 10000 & 2000 & 4 & 400 \\
{ModelNet} & Image & 12311 & 12311 & 100 & 40 & 800 \\
{Mushroom} & Image & 8124 & 298 & 22 & 2 & 40\\
\bottomrule
\end{tabular}
\label{tab1}
\end{table}

We assess the performance of MM-HGNN using the following real-world datasets.
The dataset statistics are provided in Table~\ref{tab1}, as shown in the first six columns.

\noindent\textbf{CA-Cora}\footnote{https://people.cs.umass.edu/mccallum/data.html}: We construct a coauthorship hypergraph, where each node represents a publication and each hyperedge contains all publications by the same author \cite{yadati2019hypergcn}.

\noindent\textbf{Citeseer}\footnote{https://linqs.soe.ucsc.edu/data}: We form a cocitation hypergraph, where each node represents a paper and each hyperedge contains all papers cited by a particular paper.

\noindent\textbf{20news}\footnote{https://archive.ics.uci.edu/ml/index.php\label{foot2}} : We create a news-word hypergraph, where each node represents a news article and each hyperedge represents all articles containing the same word. The construction process follows \cite{chien2021you}.

\noindent\textbf{Reuters} \cite{10.1145/2809695.2809718}: We build \(k\)-uniform hypergraphs in the text domain, where nodes represent words and hyperedges capture higher-order relationships among co-occurring words.

\noindent\textbf{ModelNet}\footnote{https://modelnet.cs.princeton.edu/}: We form a nearest-neighbor hypergraph, where each node represents a CAD model and each hyperedge is constructed based on Multi-view Convolutional Neural Network (MVCNN) features with ten nearest neighbors.

\noindent\textbf{Mushroom}\textsuperscript{\ref{foot2}}: We establish a mushroom-feature hypergraph, where each node represents a mushroom and each hyperedge contains all data points with the same categorical features.

\subsection{Baselines}

We compare OMA-HGNN with the following nine baseline methods:

\noindent\textbf{HGNN} \cite{2019Hypergraph}: This hypergraph convolutional network approximates hypergraph convolution using truncated Chebyshev polynomials.

\noindent\textbf{HyperGCN} \cite{yadati2019hypergcn}: This method transforms the hypergraph into subgraphs and introduces mediator nodes to convert the hypergraph into a standard graph for graph convolution.

\noindent\textbf{HyperSAGE} \cite{2020HyperSAGE}: This model employs a two-stage message passing mechanism to explore hypergraph structures. 

\noindent\textbf{UniGNN} \cite{2021UniGNN}: This unified framework elucidates the message passing processes in both graph and hypergraph neural networks. We present the results of UniGCN and UniGAT.

\noindent\textbf{AllSet} \cite{chien2021you}: This supervised hypergraph model propagates information through a multiset function learned by Deep Sets \cite{zaheer2017deep} and Set Transformer \cite{lee2019set}.

\noindent\textbf{HyperAtten} \cite{bai2021hypergraph}: This model incorporates hypergraph convolution and attention mechanisms to enhance representation learning. 

\noindent\textbf{DPHGNN} \cite{10.1145/3637528.3672047}: This model employs a dual-perspective hypergraph neural network that combines topology-aware spectral inductive biases with spatial message passing via equivariant operator learning.

\noindent\textbf{KHGNN} \cite{11063418}: This model introduces a kernelized aggregation strategy that adaptively combines mean- and max-based aggregations via learnable parameters.

\subsection{Experimental setup}
Each dataset is divided into training and test sets, with ten random splits generated to ensure robustness.
The training set split size can be found in the last column of Table~\ref{tab1}. 
All models are trained and evaluated on the same splits to maintain consistency. For OMA-HGNN, the validation set for training the MWN model is selected with minimal
attention bias from the test set, matching the size of the training set. 
The model is trained by minimizing the cross-entropy loss with the Adam optimizer.  The hidden layer dimensionality is set to 64, and the model consists of 2 layers. Additionally, the hyperparameters for all baselines are configured according to the optimal settings specified in their original publications.

\subsection{Experimental results}

\subsubsection{Comparative analysis}
To validate the effectiveness of the proposed OMA-HGNN, we conduct comparative experiments against nine  state-of-the-art baselines on six benchmark datasets. The comparison results are presented in Table \ref{tab2}. The highest accuracy on each dataset is highlighted in bold, while the second-highest results are underlined for clarity.

\begin{table*}[htbp]
\centering
\caption{Comparison results of node classification accuracy(\%).}
\resizebox{1\textwidth}{!}{
\begin{tabular}{cccccccc}
\toprule
Dataset & CA-Cora & Citeseer & 20news & Reuters & ModelNet & Mushroom  \\
\midrule
HGNN       & $75.7 \pm 1.0$ & $64.8 \pm 1.0$ & $76.5 \pm 1.7$ &    $92.2 \pm 0.6$           & \underline{$94.5 \pm 0.1$}  & $94.5 \pm 1.9$ \\
HyperGCN   & $70.1 \pm 0.7$ & $63.4 \pm 1.1$ & $70.3 \pm 1.6$      &   $92.1  \pm 0.6$             & $93.8 \pm 1.7$  & $89.2 \pm 2.7$ \\
HyperSAGE  & $58.4 \pm 1.4$ & $59.6 \pm 1.3$ & $68.5 \pm 1.8$      &       $90.5 \pm 0.7$         & $89.6 \pm 1.0$  & $90.8 \pm 2.3$ \\
UniGCN     & $75.3 \pm 1.2$ & $63.9 \pm 1.4$ & $78.2 \pm 0.5$ &    $89.9  \pm 1.1$            & $94.0 \pm 0.2$ & $94.6 \pm 2.0$ \\
UniGAT     & $75.7 \pm 1.2$ & $64.1 \pm 1.5$ & {$78.0 \pm 0.7$} &    $89.5  \pm 1.7$            & $93.5 \pm 0.1$  & $94.2 \pm 1.6$ \\
AllSet     & $58.4 \pm 1.5$ & $57.1 \pm 1.5$ & $68.7 \pm 1.6$      &      $91.2  \pm 0.3$          & $90.2 \pm 0.4$  & $90.4 \pm 2.3$ \\
HyperAtten & $65.9 \pm 0.8$ & $56.2 \pm 3.3$ & $75.9 \pm 1.1$      &   $86.8 \pm 0.9$             & $91.9 \pm 0.1$  & $87.2 \pm 1.9$ \\
DPHGNN     & $75.0 \pm 2.0$ & \underline{$65.7 \pm 4.2$} & \bm{$80.6 \pm 0.3$}      & \bm{$94.6 \pm 0.2$}              & $90.8 \pm 0.2$  & \underline{$95.3 \pm 1.2$} \\
KHGNN      & \underline{$75.8 \pm 0.7$} & $63.6 \pm 1.0$ & $77.5 \pm 0.3$ & \underline{$93.7 \pm 0.2$} & \bm{$94.8 \pm 0.2$} & $90.1 \pm 0.8$ \\
OMA-HGNN   & \bm{$78.5 \pm 1.3$} & \bm{$69.5 \pm 2.2$} & \underline{$79.6 \pm 0.7$} &   $93.4 \pm 0.4$             & \bm{$94.8 \pm 0.1$} & \bm{$96.1 \pm 1.5$} \\
\bottomrule
\end{tabular}
}
\label{tab2}
\end{table*}

\begin{figure*}[htbp]
  \centering
  \subfigure[]{\includegraphics[width=0.45\linewidth]{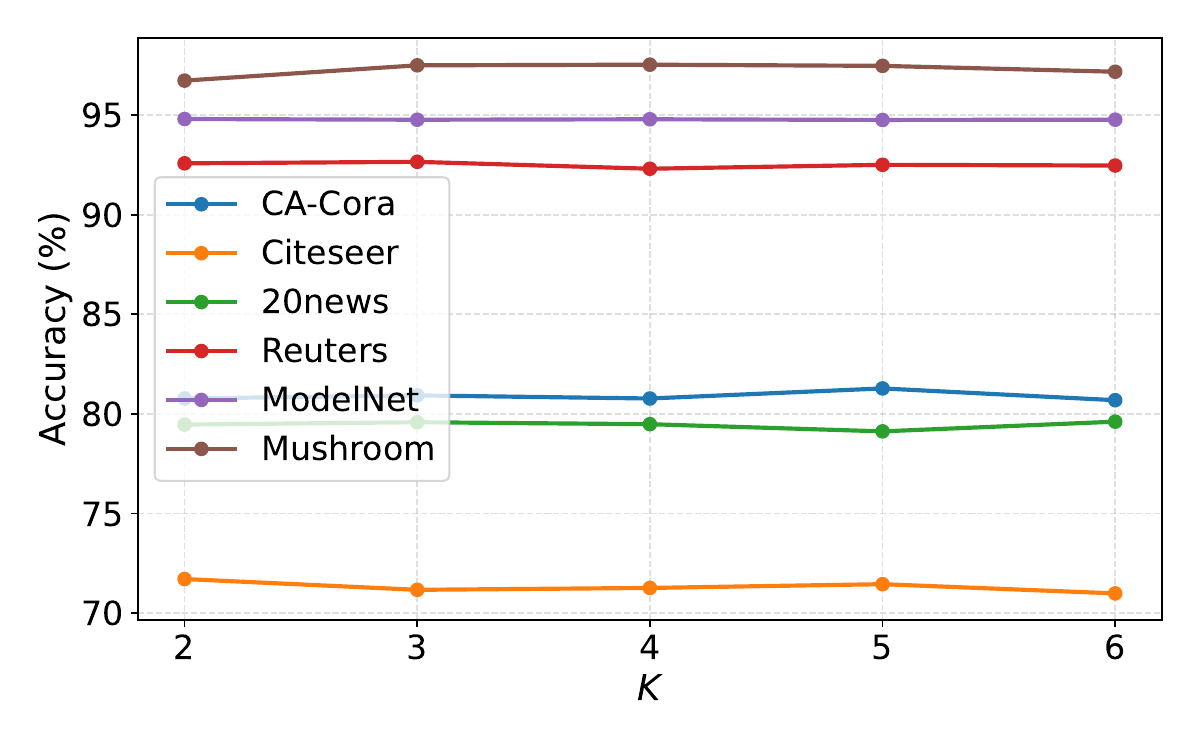}}
  \subfigure[]{\includegraphics[width=0.45\linewidth]{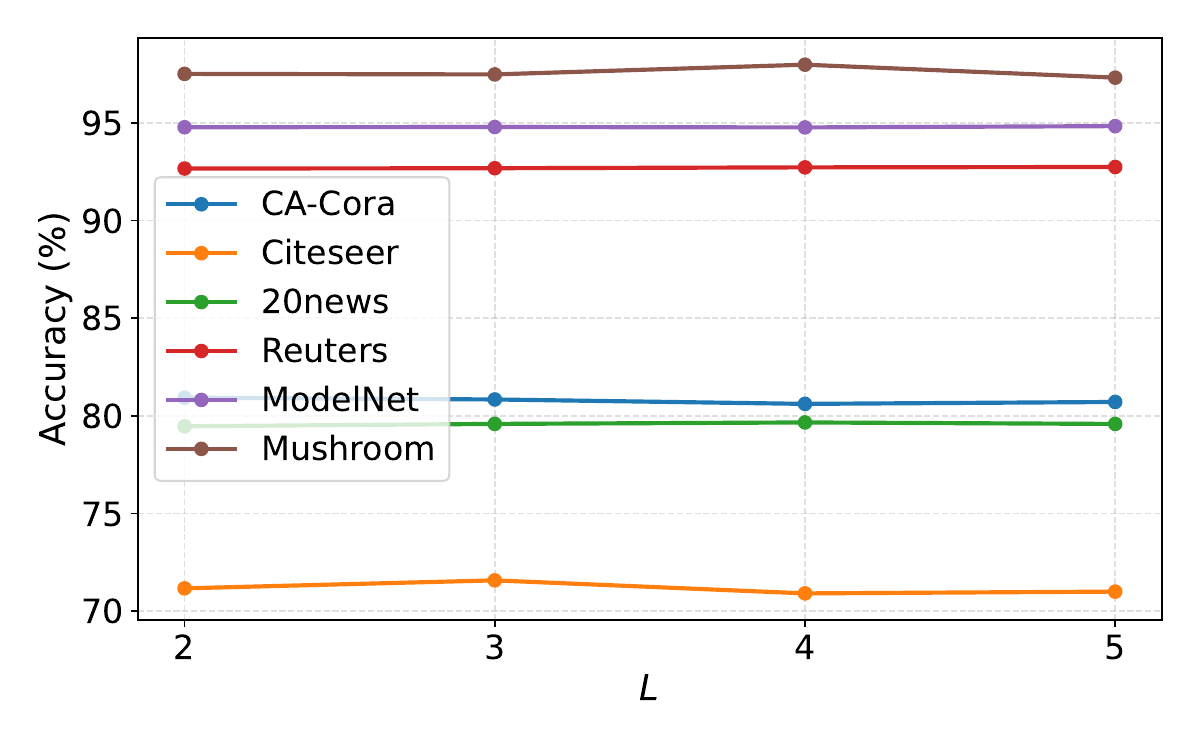}}
  \caption{Overall comparison of OMA-HGNN under different hyperparameter settings: (a) different values of overlap levels $K$; (b) different hidden layers of  MT-MWN $L$.}
  \label{fig:accuracy}
\end{figure*}

Table \ref{tab2} presents the comparison results of node classification accuracy across six datasets. The highest accuracy on each dataset is highlighted in bold, while the second-highest results are underlined for clarity.
As observed, OMA-HGNN consistently achieves the best performance across CA-Cora, Citeseer, ModelNet and Mushroom datasets, demonstrating its superior capability in learning discriminative node representations. In particular, it yields significant improvements of approximately 4\% and 3\%, on the Citeseer, CA-Cora datasets, respectively. These gains are largely attributed to its overlap-aware meta-learning attention mechanism, which effectively captures higher-order dependencies in hypergraphs. Moreover, the improvements of OMA-HGNN are especially pronounced on network datasets (e.g., Citeseer, CA-Cora), likely because their hypergraph construction leverages explicit structural information, offering rich relational cues for representation learning.

\begin{figure*}[htbp]
  \centering
  \subfigure[]{\includegraphics[width=0.31\linewidth]{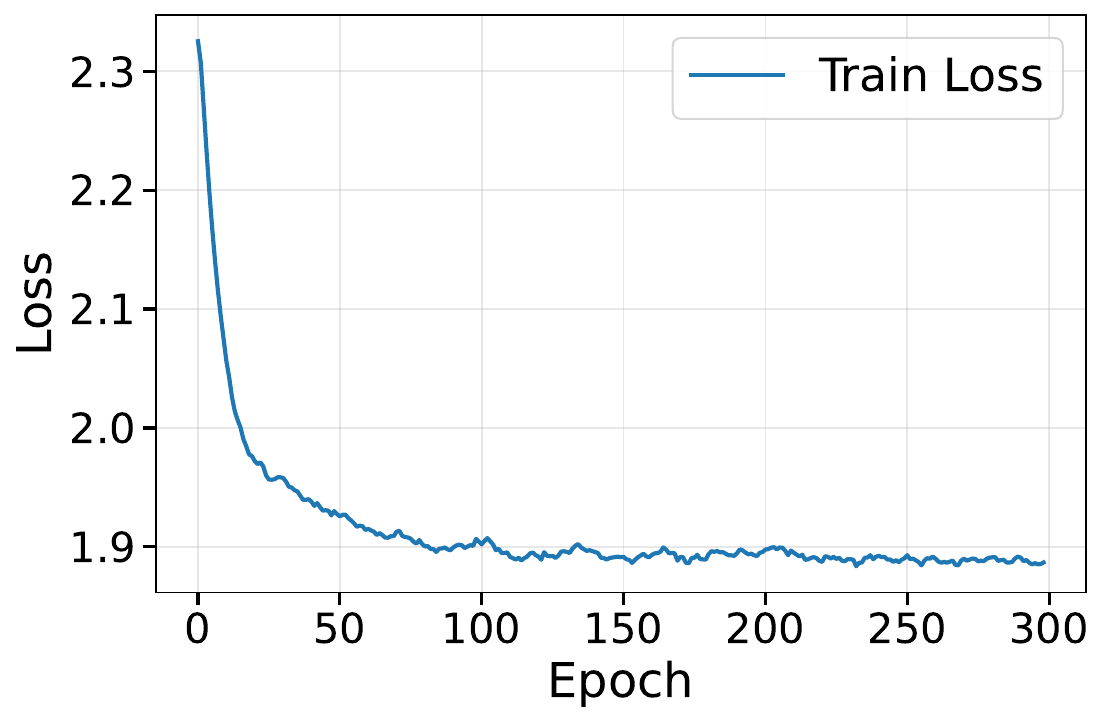}}
  \subfigure[]{\includegraphics[width=0.31\linewidth]{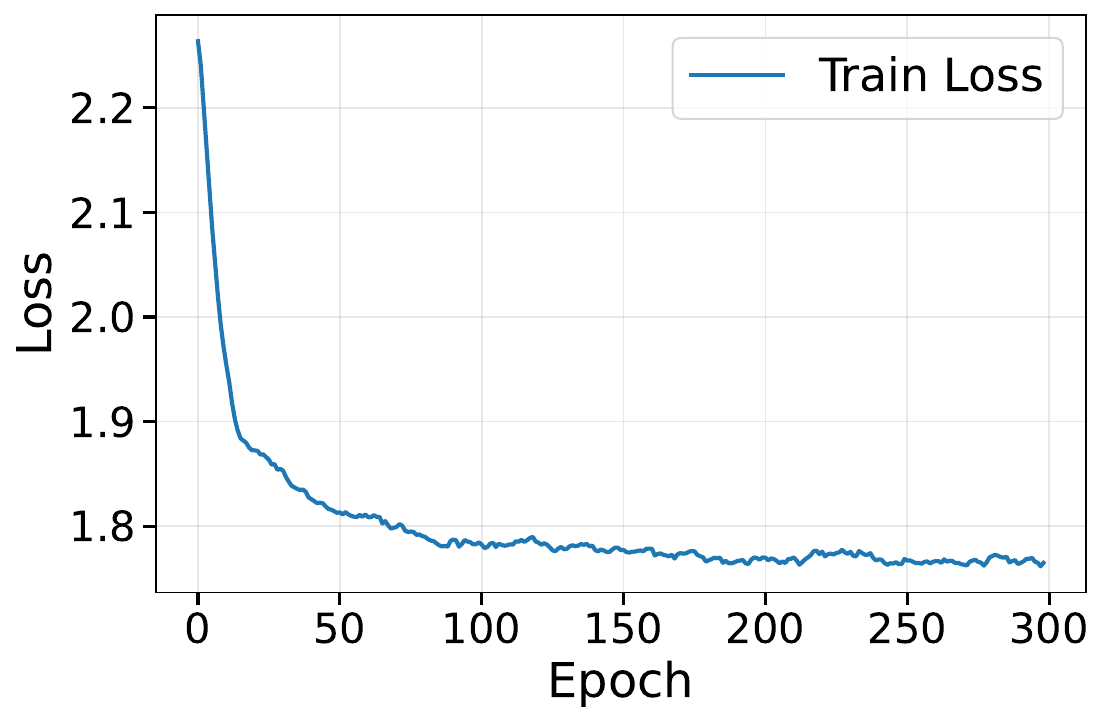}}
    \subfigure[]{\includegraphics[width=0.31\linewidth]{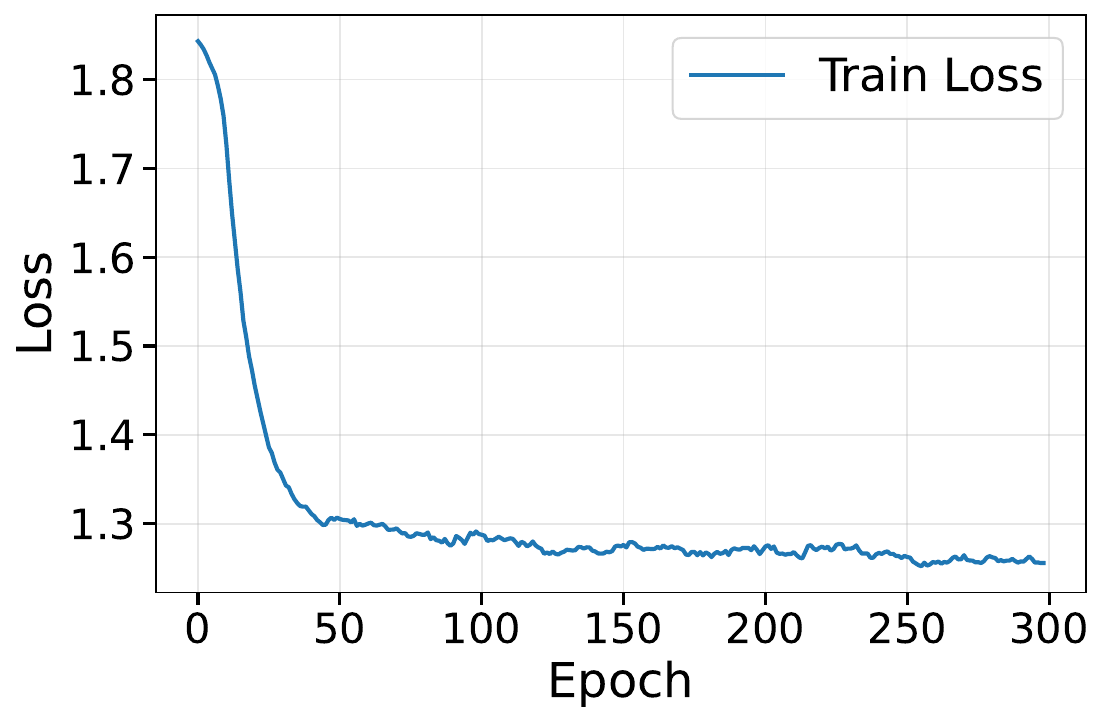}}
  \subfigure[]{\includegraphics[width=0.31\linewidth]{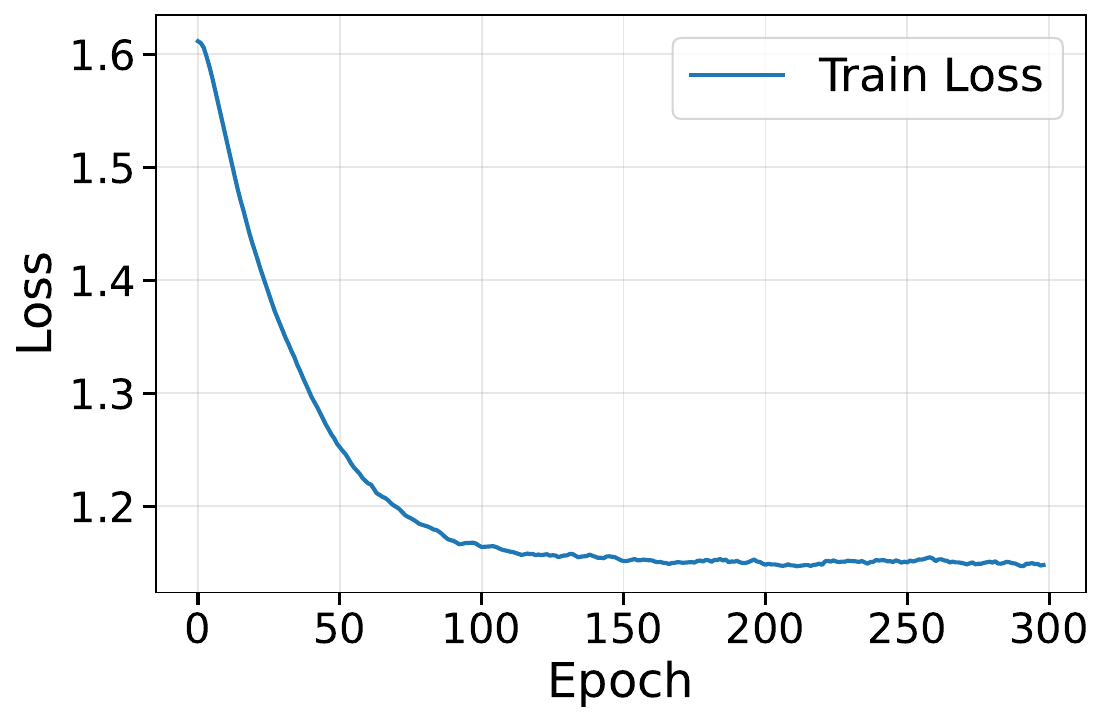}}
    \subfigure[]{\includegraphics[width=0.31\linewidth]{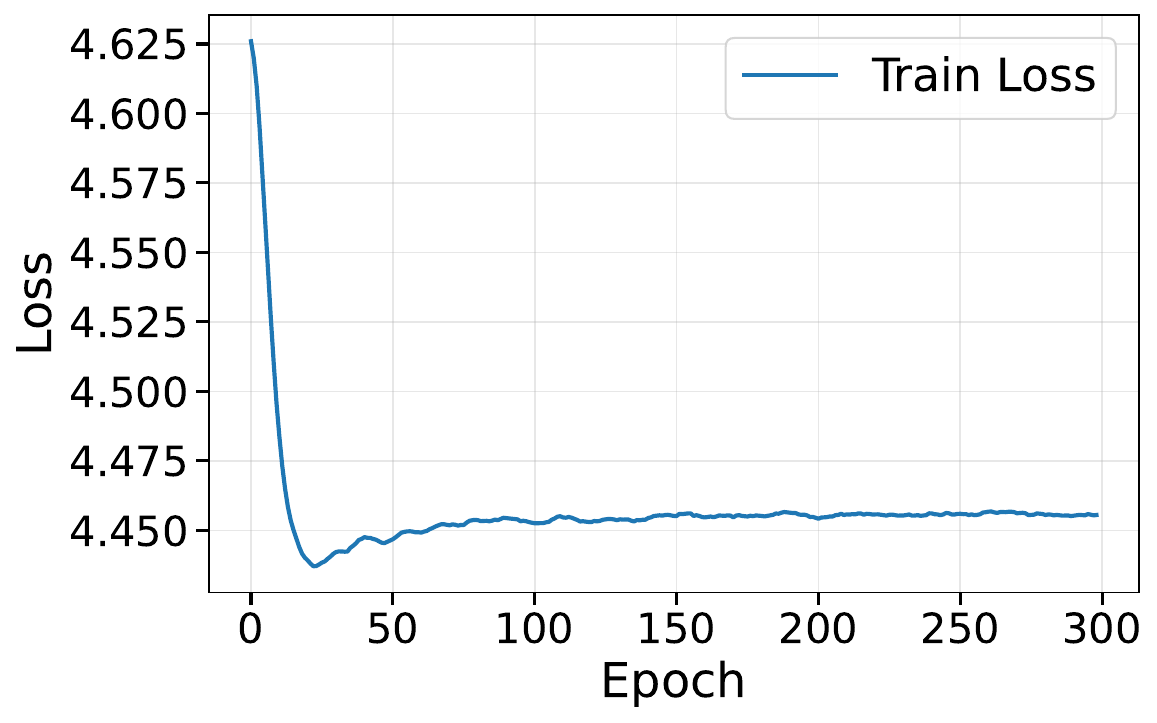}}
  \subfigure[]{\includegraphics[width=0.31\linewidth]{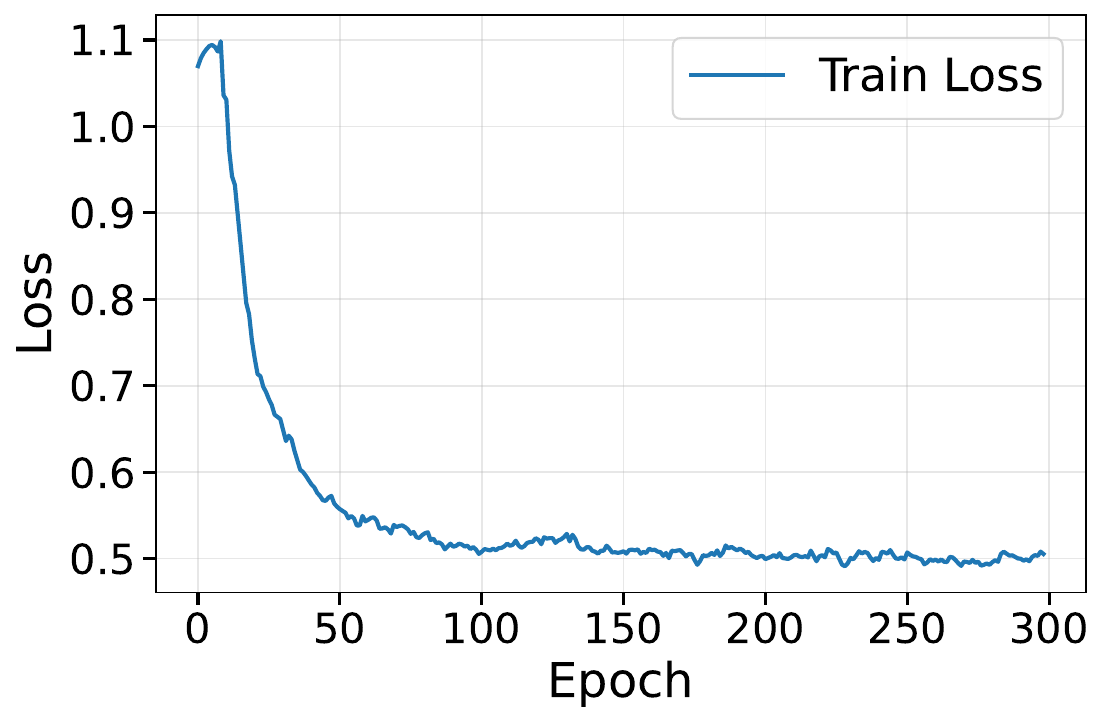}}

  \caption{Training loss curves on benchmark datasets: (a) CA-Cora, (b) Citeseer, (c) 20news, (d) Reuters, (e) ModelNet and (f) Mushroom. }
  \label{fig:loss}
\end{figure*}

On 20news and Reuters, OMA-HGNN attains accuracies of 79.6\% and 93.4\%, slightly lower than DPHGNN. This difference can be explained by dataset-specific characteristics: DPHGNN is particularly suited for handling high-dimensional sparse text data.  Nevertheless, given that the performance gaps are within 0.2\% and OMA-HGNN achieves the best results on the remaining four datasets, it still demonstrates robust and consistent advantages across the majority of tasks.

\subsubsection{Parameters sensitivity}

To better understand the robustness and practicality of OMA-HGNN, we further conduct parameter sensitivity analysis.
By systematically varying these parameters, we aim to examine whether OMA-HGNN maintains stable accuracy across a wide range of settings and to identify reasonable choices that balance model complexity and generalization ability.

Figure~\ref{fig:accuracy} (a) compares the effect of different overlap levels $K$ on classification accuracy. Across all six datasets, the curves exhibit highly consistent trends: performance remains stable overall, with slight improvements when $K$ is set to $3$ or $4$. This indicates that OMA-HGNN is robust to the choice of partition granularity. This indicates that a moderate number of partitions provides sufficient diversity to distinguish nodes with different overlap levels while still keeping enough samples per group for stable training, whereas overly coarse or overly fine partitions can dilute the effectiveness of the meta-learning process.

Figure~\ref{fig:accuracy} (b) shows the impact of different hidden layers $L$ of MT-MWN. The accuracy generally increases when the hidden layers grow from $2$ to $3$, but further depth yields little or even negative gain. This suggests that a shallow-to-moderate architecture already captures the necessary task-specific weighting patterns, while deeper structures introduce redundant complexity and risk overfitting. Taken together, these results demonstrate that OMA-HGNN consistently achieves strong performance under reasonable settings, confirming both the robustness and the practicality of the proposed design.

\subsubsection{Visualization of training loss}

Figure~\ref{fig:loss} presents the training loss trajectories on benchmark datasets.
Across all datasets, the loss curves exhibit a clear and monotonic downward trend, demonstrating that OMA-HGNN converges rapidly within a relatively small number of epochs.
The convergence is stable, with only minor fluctuations around the decreasing trajectory, which indicates that the overlap-aware weighting mechanism effectively suppresses noisy gradients and prevents oscillatory updates.
We also observe that the convergence speed is consistent across heterogeneous datasets, and the plateau levels of the loss are similar, suggesting that the model is insensitive to dataset-specific idiosyncrasies.

\section{Conclusion}\label{s6}

In this paper, we propose a novel framework named Overlap-aware Meta-learning Attention for Hypergraph Neural Networks (OMA-HGNN).
First, we weight and linearly combine the losses from SS-based and FS-based attention mechanisms with weighted factors to form the total loss of the external HGNN model.
To determine the weighted factors, we categorize nodes into distinct tasks based on their overlap levels and develop a Multi-Task MWN (MT-MWN).
Then, we jointly train the external HGNN model and the internal MT-MWN model.
Specifically, the losses of the external model at each step are used to update the internal internal model's parameters, which determines the weighted factors. In turn, the weighted factors are fed back into the external model to update its parameters and obtain the losses for the internal internal model at next step.
Finally, we present the iterative algorithm for OMA-HGNN and conduct a theoretical convergence analysis. Extensive experiments across multiple datasets for node classification  demonstrate that the proposed method outperforms state-of-the-art methods, highlighting its effectiveness.

\section*{Acknowledgments}

This work was supported by the Natural Science Foundation of China under Grant No.  12531019 and the Science and Technology Commission of Shanghai Municipality under Grant No. 22JC1401401.

\bibliography{ref}

\newpage
\appendix
\section{Proposition 1}\label{7.1}

\begin{proposition}
Let $\Theta^{(t)}$ represent the parameters of the internal model at iteration $t$, and let $\hat{\mathbf{w}}^{(t+1)}(\Theta)$ be the predicted weights as a function of $\Theta$. The update rule for the internal model parameters $\Theta$ is given by:

\begin{equation}
\label{eq12}
\Theta^{(t+1)} = \Theta^{(t)} - \lambda_2 \frac{1}{m} \sum_{i=1}^m \left( \nabla_{\Theta} L_{i,\text{meta}}^1(\hat{\mathbf{w}}^{(t+1)}(\Theta)) + \nabla_{\Theta} L_{i,\text{meta}}^2(\hat{\mathbf{w}}^{(t+1)}(\Theta)) \right) \Big|_{\Theta^{(t)}}
\end{equation}

For simplicity, let $\mathcal{V}_j^{(t)}(\Theta)$ denote the auxiliary representation defined as:

\[
\mathcal{V}_j^{(t)}(\Theta) = \mathcal{V}(L_{j,\text{train}}^1(\mathbf{w}^{(t)}), L_{j,\text{train}}^2(\mathbf{w}^{(t)}); \Theta),
\]
where $L_{j,\text{train}}^1(\mathbf{w}^{(t)})$ and $L_{j,\text{train}}^2(\mathbf{w}^{(t)})$ represent the training loss for the $j$-th sample.  
The gradient of the loss function with respect to $\Theta$ can be derived by:

\begin{equation}
\label{eq13}
\begin{aligned}
& \frac{1}{m} \sum_{i=1}^m \left( \nabla_{\Theta} L_{i,\text{meta}}^1(\hat{\mathbf{w}}^{(t+1)}(\Theta)) + \nabla_{\Theta} L_{i,\text{meta}}^2(\hat{\mathbf{w}}^{(t+1)}(\Theta)) \right) \Big|_{\Theta^{(t)}} \\
&= \frac{1}{m} \sum_{i=1}^m \left( \frac{\partial L_{i,\text{meta}}^1(\hat{\mathbf{w}}^{(t+1)}(\Theta))}{\partial \hat{\mathbf{w}}^{(t+1)}(\Theta)} + \frac{\partial L_{i,\text{meta}}^2(\hat{\mathbf{w}}^{(t+1)}(\Theta))}{\partial \hat{\mathbf{w}}^{(t+1)}(\Theta)} \right)\\
& \quad\quad \times \sum_{j=1}^n \frac{\partial \hat{\mathbf{w}}^{(t+1)}(\Theta)}{\partial \mathcal{V}_j^{(t)}(\Theta)} \frac{\partial \mathcal{V}_j^{(t)}(\Theta)}{\partial \Theta} \Big|_{\Theta^{(t)}}
\end{aligned}
\end{equation}

Recall  the parameter update equation  of the external model as follows:
$$
\label{eq11}
\begin{aligned}
 \mathbf{w}^{(t+1)} = \mathbf{w}^{(t)}-
 \left.\lambda_{1} \sum_{i=1}^n \left(\alpha_i^{(t)}
\nabla_{\mathbf{w}}L_{i,train}^1(\mathbf{w}^{(t)})+ \beta_i^{(t)} \nabla_{\mathbf{w}}L_{i,train}^2(\mathbf{w}^{(t)})\right)\right|_{\mathbf{w}^{(t)}}
\end{aligned}
$$

For the sake of proof, we assume that for each sample, the following holds:
$$
\alpha_i^{(t)} + \beta_i^{(t)} = 1, \quad \alpha_i^{(t)} = \mathcal{V}\left(L_{i,\text{train}}^1(\mathbf{w}^{(t)}), L_{i,\text{train}}^2(\mathbf{w}^{(t)}); \Theta\right).
$$

Let $L_{i,\text{meta}}(\hat{\mathbf{w}}^{(t+1)}(\Theta)) = L_{i,\text{meta}}^1(\hat{\mathbf{w}}^{(t+1)}(\Theta)) + L_{i,\text{meta}}^2(\hat{\mathbf{w}}^{(t+1)}(\Theta))$ represent the combined meta-loss function.
And $L_{j,\text{train}-}(\mathbf{w}) = L_{j,\text{train}}^1(\mathbf{w}) - L_{j,\text{train}}^2(\mathbf{w})$ represents the combined training loss for the $j$-th sample.
Substituting (\ref{eq13}) and these formulas into equation (\ref{eq12}), we obtain the update rule for $\Theta$:

\begin{equation}
\Theta^{(t+1)} = \Theta^{(t)} + \lambda_1 \lambda_2 \sum_{j=1}^n \left( \frac{1}{m} \sum_{i=1}^m G_{ij} \right) \frac{\partial \mathcal{V}_j^{(t)}(\Theta)}{\partial \Theta} \Big|_{\Theta^{(t)}}
\end{equation}
where $G_{ij}$ is defined as:

$$
G_{ij} = \frac{\partial L_{i,\text{meta}}(\hat{\mathbf{w}}^{(t+1)}(\Theta))}{\partial \hat{\mathbf{w}}^{(t+1)}(\Theta)} \Big|_{\hat{\mathbf{w}}^{(t+1)}(\Theta)}^T \frac{\partial L_{j,\text{train}-}(\mathbf{w})}{\partial \mathbf{w}} \Big|_{\mathbf{w}^{(t)}}.
$$
%
%
\end{proposition}

\section{Lemma 1}

\begin{lemma}\label{le}
Suppose the internal model loss function (\ref{eq:ml})\emph{} is Lipschitz smooth with constant $M$, and $\mathcal{V}(\cdot, \cdot ; \Theta)$ is differential with a $\delta$-bounded gradient and twice differential with its Hessian bounded by $\mathcal{B}$, and the loss function $\ell_1$ and $\ell_2$ have $\rho$-bounded gradient with respect to training/validation set. Then the gradient of $\Theta$ with respect to the internal model loss is Lipschitz continuous.

\end{lemma}

\begin{proof}
The gradient of $\Theta$ with respect to the meta loss can be written by:

\begin{equation}
\label{eq19}
\begin{aligned}
& \left. \left(\nabla_{\Theta}
L_{i,\text{meta}}^1(\hat{\mathbf{w}}^{(t+1)}(\Theta))+\nabla_{\Theta}
L_{i,\text{meta}}^2(\hat{\mathbf{w}}^{(t+1)}(\Theta))\right)\right|_{\Theta^{(t)}} \\
= & -\left.\lambda_1 \sum_{j=1}^n\left(\left.\left. \frac{\partial L_{i,meta}(\hat{\mathbf{w}})}{\partial \hat{\mathbf{w}}}\right|_{\hat{\mathbf{w}}^{(t+1)}(\Theta)}^T  \frac{\partial L_{j,\text{train}-}(\mathbf{w})}{\partial \mathbf{w}}\right|_{\mathbf{w}^{(t)}}\right) \frac{\partial \mathcal{V}_j^{(t)}(\Theta)}{\partial \Theta}\right|_{\Theta^{(t)}}\\
= & -\left.\lambda_1\sum_{j=1}^n\left( G_{i j}\right) \frac{\partial \mathcal{V}_j^{(t)}(\Theta)}{\partial \Theta}\right|_{\Theta^{(t)}}.
\end{aligned}
\end{equation}
Taking gradient of $\Theta$ in both sides of Eq.(\ref{eq19}), we have
\begin{equation}
\label{eq20}
\begin{aligned}
& \left. \left(\nabla_{\Theta^2}^2
L_{i,\text{meta}}^1\left(\hat{\mathbf{w}}^{(t+1)}(\Theta)\right)+\nabla_{\Theta^2}^2
L_{i,\text{meta}}^2\left(\hat{\mathbf{w}}^{(t+1)}(\Theta)\right)\right)\right|_{\Theta^{(t)}}\\ = & -\lambda_1 \sum_{j=1}^n\left[\left.\left.\frac{\partial}{\partial \Theta}\left(G_{i j}\right)\right|_{\Theta^{(t)}} \frac{\partial \mathcal{V}_j^{(t)}(\Theta)}{\partial \Theta}\right|_{\Theta^{(t)}}+\left.\left(G_{i j}\right) \frac{\partial^2 \mathcal{V}_j^{(t)}(\Theta)}{\partial \Theta^2}\right|_{\Theta^{(t)}}\right].
\end{aligned}
\end{equation}

Since
the internal model loss function (\ref{eq:ml}) is Lipschitz smooth with constant $M$:
$$\left\|\left.\frac{\partial^2 L_{i,\text{meta}}^1(\hat{\mathbf{w}})}{\partial \hat{\mathbf{w}}^2}\right|_{\hat{\mathbf{w}}^{(t+1)}(\Theta)}+\left.\frac{\partial^2 L_{i,\text{meta}}^1(\hat{\mathbf{w}})}{\partial \hat{\mathbf{w}}^2}\right|_{\hat{\mathbf{w}}^{(t+1)}(\Theta)}\right\| \leq M.$$

Suppose the dimension of $\mathbf{w}$ is $n\times n'$. Since $\left\|\left.\frac{\partial L_{j,tr}^1(\mathbf{w})}{\partial \mathbf{w}}\right|_{\mathbf{w}^{(t)}}\right\| \leq \rho$ and
$\left\|\left.\frac{\partial L_{j,tr}^2(\mathbf{w})}{\partial \mathbf{w}}\right|_{\mathbf{w}^{(t)}}\right\| \leq \rho$,
according to Cauchy-Schwarz inequality,
$$\left\|\left.\frac{\partial L_{j,\text{train}-}(\mathbf{w})}{\partial \mathbf{w}}\right|_{\mathbf{w}^{(t)}}\right\|  = \left\|\left.\left(\frac{\partial L_{j,\text{train}}^1(\mathbf{w})}{\partial \mathbf{w}}-\frac{\partial L_{j,\text{train}}^1(\mathbf{w})}{\partial \mathbf{w}}\right)\right|_{\mathbf{w}^{(t)}}\right\| \leq \sqrt{n\cdot n'}\rho. $$
And $\mathcal{V}(\cdot, \cdot ; \Theta)$ is differential with a $\delta$-bounded gradient, i.e., $\left\|\left.\frac{\partial \mathcal{V}_j(\Theta)}{\partial \Theta}\right|_{\Theta^{(t)}}\right\| \leq \delta$.

For the first term in (\ref{eq20}), we have

\begin{equation}
\label{eq21}
\begin{aligned}
& \left\|\left.\left.\frac{\partial}{\partial \Theta}\left(G_{i j}\right)\right|_{\Theta^{(t)}} \frac{\partial \mathcal{V}_j^{(t)}(\Theta)}{\partial \Theta}\right|_{\Theta^{(t)}}\right\| \\
\leq & \delta\left\|\left.\frac{\partial}{\partial \Theta}\left(G_{i j}\right)\right|_{\Theta^{(t)}}\right\|\\
= & \delta\left\|\left.\left.\frac{\partial}{\partial \hat{\mathbf{w}}}\left(\left.\frac{\partial L_{i,meta}(\hat{\mathbf{w}})}{\partial \Theta}\right|_{\Theta^{(t)}}\right)\right|_{\hat{\mathbf{w}}^{(t+1)}(\Theta)} ^T \frac{\partial L_{j,\text{train}-}(\mathbf{w})}{\partial \mathbf{w}}\right|_{\mathbf{w}^{(t)}}\right\| \\
= & \delta\left\|\left.\frac{\partial}{\partial \hat{\mathbf{w}}}\left(\left.\left.\left.\frac{\partial L_{i,meta}(\hat{\mathbf{w}})}{\partial \hat{\mathbf{w}}}\right|_{\hat{\mathbf{w}}^{(t+1)}(\Theta)}(-\lambda_1) \sum_{k=1}^n \frac{\partial L_{k,\text{train}-}(\mathbf{w})}{\partial \mathbf{w}}\right|_{\mathbf{w}^{(t)}} \frac{\partial \mathcal{V}_k^{(t)}(\Theta)}{\partial \Theta}\right|_{\Theta^{(t)}}\right)\right|_{\hat{\mathbf{w}}^{(t+1)}(\Theta)} ^T \times \mathcal{A} \right\| \\
= & \delta\left\|\left.\left(\left.\left.\left.\frac{\partial^2 L_{i,meta}(\hat{\mathbf{w}})}{\partial \hat{\mathbf{w}}^2}\right|_{\hat{\mathbf{w}}^{(t+1)}(\Theta)}(-\lambda_1) \sum_{k=1}^n \frac{\partial L_{k,\text{train}-}(\mathbf{w})}{\partial \mathbf{w}}\right|_{\mathbf{w}^{(t)}} \frac{\partial \mathcal{V}_k^{(t)}(\Theta)}{\partial \Theta}\right|_{\Theta^{(t)}}\right)\right|_{\hat{\mathbf{w}}^{(t+1)}(\Theta)} ^T \times \mathcal{A}  \right\| \\
\leq & \lambda_1 n' M  n^2\rho^2 \delta^2
\end{aligned}
\end{equation}
where $\mathcal{A}=\left.\frac{\partial L_{j,\text{train}-}(\mathbf{w})}{\partial \mathbf{w}}\right|_{\mathbf{w}^{(t)}}$.
Since $\mathcal{V}(\cdot, \cdot ; \Theta)$ is twice differential with its Hessian bounded by $\mathcal{B}$, we can obtain that $\left\|\left.\frac{\partial^2 \mathcal{V}_j^{(t)}(\Theta)}{\partial \Theta^2}\right|_{\Theta^{(t)}}\right\| \leq \mathcal{B}$.

From the assumptions $\left\|\left.\frac{\partial L_{i,\text{meta}}^1(\hat{\mathbf{w}})}{\partial \hat{\mathbf{w}}}\right|_{\hat{\mathbf{w}}^{(t+1)}(\Theta)} ^T\right\| \leq \rho, \left\|\left.\frac{\partial L_{i,\text{meta}}^2(\hat{\mathbf{w}})}{\partial \hat{\mathbf{w}}}\right|_{\hat{\mathbf{w}}^{(t+1)}(\Theta)} ^T\right\| \leq \rho$, we have for the second term  in (\ref{eq20}):
\begin{equation}
\label{eq22}
\begin{aligned}
\left\|\left.\left(G_{i j}\right) \frac{\partial^2 \mathcal{V}_j^{(t)}(\Theta)}{\partial \Theta^2}\right|_{\Theta^{(t)}}\right\|&=\left\|\left.\left.\left.\frac{\partial L_{i,meta}(\hat{\mathbf{w}})}{\partial \hat{\mathbf{w}}}\right|_{\hat{\mathbf{w}}^{(t+1)}(\Theta)} ^T \frac{\partial L_{j,\text{train}-}(\mathbf{w})}{\partial \mathbf{w}}\right|_{\mathbf{w}^{(t)}} \frac{\partial^2 \mathcal{V}_j^{(t)}(\Theta)}{\partial \Theta^2}\right|_{\Theta^{(t)}}\right\| \\
 & \leq 2\sqrt{n n'}\mathcal{B} \rho^2.
\end{aligned}
\end{equation}

Combining the above two inequalities Eqs.(\ref{eq21}) and (\ref{eq22}), then we have

\begin{equation}
\begin{aligned}
& \quad \left\| \left. \left(\nabla_{\Theta^2}^2
L_{i,\text{meta}}^1\left(\hat{\mathbf{w}}^{(t+1)}(\Theta)\right)+\nabla_{\Theta^2}^2
L_{i,\text{meta}}^2\left(\hat{\mathbf{w}}^{(t+1)}(\Theta)\right)\right)\right|_{\Theta^{(t)}} \right\| \\
& \leq \lambda_1 \rho^2 (\lambda_1 n' M  n^2 \delta^2 +\sqrt{n n'}\mathcal{B}).
\end{aligned}
\end{equation}

Let $M_V=\lambda_1 \rho^2 (\lambda_1 n' M  n^2 \delta^2 +\sqrt{n n'}\mathcal{B})$,  based on Lagrange mean value theorem, for all $ \Theta_1, \Theta_2$ we have
\begin{equation}
\qquad\left\|\nabla \mathcal{L}^{{\text{meta} }}\left(\hat{\mathbf{w}}^{(t+1)}\left(\Theta_1\right)\right)-\nabla \mathcal{L}^{ {\text{meta} }}\left(\hat{\mathbf{w}}^{(t+1)}\left(\Theta_2\right)\right)\right\| \leq M_V\left\|\Theta_1-\Theta_2\right\|,
\end{equation}
where $ \small \nabla \mathcal{L}^{ {\text{meta} }}\left(\hat{\mathbf{w}}^{(t+1)}\left(\Theta_1\right)\right)
= \frac{1}{m} \sum\limits_{i=1}^m\left. \left(\nabla_{\Theta}
L_{i,\text{meta}}^1\left(\hat{\mathbf{w}}^{(t+1)}(\Theta)\right)+\nabla_{\Theta}
L_{i,\text{meta}}^2\left(\hat{\mathbf{w}}^{(t+1)}(\Theta)\right)\right)\right|_{\Theta_{1}} $.

Thus, the gradient of $\Theta$ with respect to the internal model loss (\ref{eq:ml}) is Lipschitz continuous.

\end{proof}

\section{Proof for Theorem \ref{th1}} \label{7.3}

%

\begin{proof}
The update equation of $\Theta$ in each iteration is as follows:
\begin{equation}
\label{eq26}
\begin{aligned}
 \Theta^{(t+1)}=\Theta^{(t)}-  \left.\lambda_2 \frac{1}{m} \sum_{i=1}^m \nabla {L}_{i,meta}\left(\hat{\mathbf{w}}^{(t+1)}\left(\Theta\right)\right)\right|_{\Theta^{(t)}},
\end{aligned}
\end{equation}
where $ \nabla {L}_{i,meta}\left(\hat{\mathbf{w}}^{(t+1)}\left(\Theta\right)\right)=\nabla_{\Theta}
L_{i,\text{meta}}^1\left(\hat{\mathbf{w}}^{(t+1)}(\Theta)\right)+\nabla_{\Theta}
L_{i,\text{meta}}^2\left(\hat{\mathbf{w}}^{(t+1)}(\Theta)\right)$.

And let $\nabla \mathcal{L}^{ {\text{meta} }}\left(\hat{\mathbf{w}}^{(t+1)}\left(\Theta\right)\right) = \frac{1}{m} \sum_{i=1}^m \nabla {L}_{i,meta}\left(\hat{\mathbf{w}}^{(t+1)}\left(\Theta\right)\right).$
Under the mini-batch $\Xi_t$ with a finite number of validation samples, the updating equation can be rewritten as:
$$
\Theta^{(t+1)}=\Theta^{(t)}-\left.\lambda_2^{(t)} \nabla \mathcal{L}^{ {\text{meta} }}\left(\hat{\mathbf{w}}^{(t+1)}\left(\Theta^{(t)}\right)\right)\right|_{\Xi_t} .
$$

Since the validation set is drawn uniformly from the entire data set, the above update equation can be written as:
$$
\Theta^{(t+1)}=\Theta^{(t)}-\lambda_2^{(t)}\left[\nabla \mathcal{L}^{ {\text{meta} }}\left(\hat{\mathbf{w}}^{(t+1)}\left(\Theta^{(t)}\right)\right)+\xi^{(t)}\right],
$$
where $\xi^{(t)}=\left.\nabla \mathcal{L}^{ {\text{meta} }}\left(\hat{\mathbf{w}}^{(t+1)}\left(\Theta^{(t)}\right)\right)\right|_{\Xi_t}-\nabla \mathcal{L}^{{\text{meta} }}\left(\hat{\mathbf{w}}^{(t+1)}\left(\Theta^{(t)}\right)\right)$.
Note that $\xi^{(t)}$ are i.i.d random variable with finite variance and samples are drawn uniformly at random. Furthermore, $\mathbb{E}\left[\xi^{(t)}\right]=0$. 

It can be observed that
\begin{equation}
\label{eq27}
\footnotesize
\begin{aligned}
& \mathcal{L}^{ {\text{meta} }}\left(\hat{\mathbf{w}}^{(t+1)}\left(\Theta^{(t+1)}\right)\right)-\mathcal{L}^{ {\text{meta} }}\left(\hat{\mathbf{w}}^{(t)}\left(\Theta^{(t)}\right)\right) \\
= & \left(\mathcal{L}^{{\text{meta} }}\left(\hat{\mathbf{w}}^{(t+1)}\left(\Theta^{(t+1)}\right)\right)-\mathcal{L}^{ {\text{meta} }}\left(\hat{\mathbf{w}}^{(t+1)}\left(\Theta^{(t)}\right)\right)\right)\\
+ & \left(\mathcal{L}^{ {\text{meta} }}\left(\hat{\mathbf{w}}^{(t+1)}\left(\Theta^{(t)}\right)\right)-\mathcal{L}^{ {\text{meta} }}\left(\hat{\mathbf{w}}^{(t)}\left(\Theta^{(t)}\right)\right)\right) .
\end{aligned}
\end{equation}

For the first term, 
by Lipschitz continuity of $\nabla_{\Theta} \mathcal{L}^{{\text{meta} }}\left(\hat{\mathbf{w}}^{(t+1)}(\Theta)\right)$ according to Lemma \ref{le}, we can deduce that:
\begin{equation}
\label{eq28}
\begin{aligned}
& \mathcal{L}^{{\text{meta} }}\left(\hat{\mathbf{w}}^{(t+1)}\left(\Theta^{(t+1)}\right)\right)-\mathcal{L}^{ {\text{meta} }}\left(\hat{\mathbf{w}}^{(t+1)}\left(\Theta^{(t)}\right)\right) \\
\leq & \left\langle\nabla_{\Theta} \mathcal{L}^{ {\text{meta} }}\left(\hat{\mathbf{w}}^{(t+1)}\left(\Theta^{(t)}\right)\right), \Theta^{(t+1)}-\Theta^{(t)}\right\rangle
+ \frac{M}{2}\left\|\Theta^{(t+1)}-\Theta^{(t)}\right\|_2^2 \\
= & \left\langle\nabla_{\Theta} \mathcal{L}^{ {\text{meta} }}\left(\hat{\mathbf{w}}^{(t+1)}\left(\Theta^{(t)}\right)\right),-\lambda_2^{(t)}\left[\nabla_{\Theta} \mathcal{L}^{{\text{meta} }}\left(\hat{\mathbf{w}}^{(t+1)}\left(\Theta^{(t)}\right)\right)+\xi^{(t)}\right]\right\rangle\\
 & \, + \frac{M {(\lambda_2^{(t)})}^2}{2}\left\|\nabla_{\Theta} \mathcal{L}^{{\text{meta} }}\left(\hat{\mathbf{w}}^{(t+1)}\left(\Theta^{(t)}\right)\right)+\xi^{(t)}\right\|_2^2 \\
= & -\left(\lambda_2^{(t)}-\frac{M {\lambda_2^{(t)}}^2}{2}\right)\left\|\nabla_{\Theta} \mathcal{L}^{{\text{meta} }}\left(\hat{\mathbf{w}}^{(t+1)}\left(\Theta^{(t)}\right)\right)\right\|_2^2\\
 & \,+ \frac{M {\lambda_2^{(t)}}^2}{2}\left\|\xi^{(t)}\right\|_2^2-\left(\lambda_2^{(t)}-M {\lambda_2^{(t)}}^2\right)\left\langle\nabla_{\Theta} \mathcal{L}^{{\text{meta} }}\left(\hat{\mathbf{w}}^{(t)}\left(\Theta^{(t)}\right)\right), \xi^{(t)}\right\rangle .
\end{aligned}
\end{equation}

For the second term, by Lipschitz smoothness of the \text{meta} loss function $\mathcal{L}^{{\text{meta} }}\left(\hat{\mathbf{w}}^{(t+1)}\left(\Theta^{(t+1)}\right)\right)$, we have
\begin{equation}
\label{eq29}
\begin{aligned}
& \mathcal{L}^{{\text{meta} }}\left(\hat{\mathbf{w}}^{(t+1)}\left(\Theta^{(t)}\right)\right)-\mathcal{L}^{ {\text{meta} }}\left(\hat{\mathbf{w}}^{(t)}\left(\Theta^{(t)}\right)\right) \\
\leq & \left\langle\nabla_{\mathbf{w}} \mathcal{L}^{ {\text{meta} }}\left(\hat{\mathbf{w}}^{(t)}\left(\Theta^{(t)}\right)\right), \hat{\mathbf{w}}^{(t+1)}\left(\Theta^{(t)}\right)-
\hat{\mathbf{w}}^{(t)}\left(\Theta^{(t)}\right)\right\rangle\\
&\,+\frac{M}{2}\left\|\hat{\mathbf{w}}^{(t+1)}\left(\Theta^{(t)}\right)-\hat{\mathbf{w}}^{(t)}\left(\Theta^{(t)}\right)\right\|_2^2 .
\end{aligned}
\end{equation}

And $
\hat{\mathbf{w}}^{(t+1)}\left(\Theta^{(t)}\right)-\hat{\mathbf{w}}^{(t)}\left(\Theta^{(t)}\right)=-\left.\lambda_1^{(t)} \nabla \mathcal{L}^{t rain}\left(\mathbf{w}^{(t)} ; \Theta^{(t)}\right)\right|_{\Psi_t},
$
where $\Psi_t$ denotes the mini-batch drawn randomly from the training dataset in the $t$-th iteration, $\nabla \mathcal{L}^{ {\text{train} }}\left(\mathbf{w}^{(t)} ; \Theta^{(t)}\right)=$ $\left.\sum\limits_{i=1}^n \alpha_i^{(t)} L_{i,\text{train}}^1(\mathbf{w}) + \beta_i^{(t)} L_{i,\text{train}}^2(\mathbf{w}) \right|_{\mathbf{w}^{(t)}}$,
and  $\alpha_i^{(t)} =\mathcal{V}\left(L_{i,\text{train}}^1(\mathbf{w}^{(t)}),L_{i,\text{train}}^2(\mathbf{w}^{(t)}); \Theta\right)$.

Since the mini-batch $\Psi_t$ is drawn uniformly at random, we can rewrite the update equation as:
$$
\hat{\mathbf{w}}^{(t+1)}\left(\Theta^{(t)}\right)=\hat{\mathbf{w}}^{(t)}\left(\Theta^{(t)}\right)-\lambda_1^{(t)}\left[\nabla \mathcal{L}^{\text{train}}\left(\mathbf{w}^{(t)} ; \Theta^{(t)}\right)+\psi^{(t)}\right],
$$
where $\psi^{(t)}=\left.\nabla \mathcal{L}^{t rain}\left(\mathbf{w}^{(t)} ; \Theta^{(t)}\right)\right|_{\Psi_t}-\nabla \mathcal{L}^{t rain}\left(\mathbf{w}^{(t)} ; \Theta^{(t)}\right)$. Note that $\psi^{(t)}$ are i.i.d. random variables with finite variance, since $\Psi_t$ are drawn i.i.d. with a finite number of samples, and thus $\mathbb{E}\left[\psi^{(t)}\right]=0, \mathbb{E}\left[\left\|\psi^{(t)}\right\|_2^2\right] \leq \sigma^2$. Thus we can rewrite (\ref{eq29}) as

\begin{equation}
\label{eq30}
\begin{aligned}
& \mathcal{L}^{{\text{meta} }}\left(\hat{\mathbf{w}}^{(t+1)}\left(\Theta^{(t)}\right)\right) - \mathcal{L}^{{\text{meta} }}\left(\hat{\mathbf{w}}^{(t)}\left(\Theta^{(t)}\right)\right) \\
\leq & \left\langle \mathcal{X}^{(t)}, -\lambda_1^{(t)} \mathcal{B}^{(t)} \right\rangle + \frac{M}{2} \left\| \lambda_1^{(t)} \mathcal{B}^{(t)} \right\|_2^2 \\
= & \frac{M {(\lambda_1^{(t)})}^2}{2} \left\| \mathcal{B}^{(t)} \right\|_2^2 - \lambda_1^{(t)} \left\langle \mathcal{X}^{(t)}, \mathcal{B}^{(t)} \right\rangle + \frac{M {(\lambda_1^{(t)})}^2}{2} \left\| \psi^{(t)} \right\|_2^2 \\
& + M {(\lambda_1^{(t)})}^2 \left\langle \mathcal{B}^{(t)}, \psi^{(t)} \right\rangle - \lambda_1^{(t)} \left\langle \mathcal{X}^{(t)}, \psi^{(t)} \right\rangle \\
\leq & \frac{M {(\lambda_1^{(t)})}^2 n^2 \rho^2}{2} + 2 \lambda_1^{(t)} \rho \left\| \mathcal{B}^{(t)} \right\|_2 + \frac{M \sigma^2 {(\lambda_1^{(t)})}^2}{2} + M {(\lambda_1^{(t)})}^2 \left\langle \mathcal{B}^{(t)}, \psi^{(t)} \right\rangle \\
& - \lambda_1^{(t)} \left\langle \mathcal{X}^{(t)}, \psi^{(t)} \right\rangle,
\end{aligned}
\end{equation}
where \( \mathcal{X}^{(t)} = \nabla_{\mathbf{w}} \mathcal{L}^{{\text{meta} }}\left(\hat{\mathbf{w}}^{(t)}\left(\Theta^{(t)}\right)\right) \) represents the gradient of the meta-loss with respect to the meta-parameters \( \hat{\mathbf{w}}^{(t)} \).
And \( \mathcal{B}^{(t)} = \nabla \mathcal{L}^{\text{train}}\left(\mathbf{w}^{(t)} ; \Theta^{(t)}\right) + \psi^{(t)} \) represents the sum of the gradient of the training loss with respect to the model parameters \( \mathbf{w}^{(t)} \) and the additional term \( \psi^{(t)} \).

The last inequality in (\ref{eq30}) holds since
\[
\left\langle \mathcal{X}^{(t)}, \mathcal{B}^{(t)} \right\rangle
\leq \left\|\mathcal{X}^{(t)}\right\|_2 \left\|\mathcal{B}^{(t)}\right\|_2
\leq 2 \rho \left\|\mathcal{B}^{(t)}\right\|_2.
\]
Since
$
\left\|\left.\frac{\partial L_{i,\text{train}}^1(\mathbf{w})}{\partial \mathbf{w}}\right|_{\mathbf{w}^{(t)}}\right\| \leq \rho, \quad \left\|\left.\frac{\partial L_{i,\text{train}}^2(\mathbf{w})}{\partial \mathbf{w}}\right|_{\mathbf{w}^{(t)}}\right\| \leq \rho, \quad \alpha_i^{(t)} + \beta_i^{(t)} = 1,
$
and
\[
\mathcal{B}^{(t)} = \left. \sum_{i=1}^{n} \alpha_i^{(t)} \frac{\partial L_{i,\text{train}}^1(\mathbf{w})}{\partial \mathbf{w}} + \beta_i^{(t)} \frac{\partial L_{i,\text{train}}^2(\mathbf{w})}{\partial \mathbf{w}} \right|_{\mathbf{w}^{(t)}},
\]
we use the triangle inequality to obtain that
$
\left\|\mathcal{B}^{(t)}\right\| \leq n \rho.
$

Combining Eq.(\ref{eq28}) and Eq.(\ref{eq30}), Eq.(\ref{eq27}) satisfies
$$
\begin{aligned}
& \mathcal{L}^{{\text{meta} }}\left(\hat{\mathbf{w}}^{(t+1)}\left(\Theta^{(t+1)}\right)\right)-\mathcal{L}^{{\text{meta} }}\left(\hat{\mathbf{w}}^{(t)}\left(\Theta^{(t)}\right)\right) \\
\leq & \frac{M {(\lambda_1^{(t)})}^2 n^2 \rho^2}{2}+2 \lambda_1^{(t)} \rho\left\|\mathcal{B}^{(t)}\right\|_2+\frac{M \sigma^2 {(\lambda_1^{(t)})}^2}{2}+M {(\lambda_1^{(t)})}^2\left\langle\mathcal{B}^{(t)}, \psi^{(t)}\right\rangle\\
& -\lambda_1^{(t)}\left\langle \mathcal{X}^{(t)}, \psi^{(t)}\right\rangle
 -\left(\lambda_2^{(t)}-\frac{M {\lambda_2^{(t)}}^2}{2}\right)\left\|\nabla_{\Theta} \mathcal{L}^{{\text{meta} }}\left(\hat{\mathbf{w}}^{(t+1)}\left(\Theta^{(t)}\right)\right)\right\|_2^2\\
 & +\frac{M {\lambda_2^{(t)}}^2}{2}\left\|\xi^{(t)}\right\|_2^2 -\left(\lambda_2^{(t)}-M {\lambda_2^{(t)}}^2\right)\left\langle\nabla_{\Theta} \mathcal{L}^{{\text{meta} }}\left(\hat{\mathbf{w}}^{(t)}\left(\Theta^{(t)}\right)\right), \xi^{(t)}\right\rangle.
\end{aligned}
$$

Since \( \mathbb{E}\left[\xi^{(t)}\right]=0, \mathbb{E}\left[\psi^{(t)}\right]=0 \) and \( \mathbb{E}\left[\left\|\xi^{(t)}\right\|_2^2\right] \leq \sigma^2 \), we rearrange the terms, and take expectations with respect to \( \xi^{(t)} \) and \( \psi^{(t)} \) on both sides
$$
\begin{aligned}
& \left(\lambda_2^{(t)}-\frac{M {\lambda_2^{(t)}}^2}{2}\right)\left\|\nabla_{\Theta} \mathcal{L}^{{\text{meta} }}\left(\hat{\mathbf{w}}^{(t+1)}\left(\Theta^{(t)}\right)\right)\right\|_2^2 \\
\leq & \frac{M {(\lambda_1^{(t)})}^2 n^2 \rho^2}{2}+2 \lambda_1^{(t)} \rho\left\|\mathcal{B}^{(t)}\right\|_2+\frac{M \sigma^2 {(\lambda_1^{(t)})}^2}{2}+\mathcal{L}^{{\text{meta} }}\left(\hat{\mathbf{w}}^{(t)}\left(\Theta^{(t)}\right)\right)\\
&-\mathcal{L}^{{\text{meta} }}\left(\hat{\mathbf{w}}^{(t+1)}\left(\Theta^{(t+1)}\right)\right)+\frac{M {(\lambda_2^{(t)})}^2}{2} \sigma^2,
\end{aligned}
$$
Summing up the above inequalities, we can obtain
{\small{
\begin{equation}
\label{eq31}
\begin{aligned}
& \sum_{t=1}^T \left(\lambda_2^{(t)}-\frac{M {\lambda_2^{(t)}}^2}{2}\right)\left\|\nabla_{\Theta} \mathcal{L}^{{\text{meta} }}\left(\hat{\mathbf{w}}^{(t+1)}\left(\Theta^{(t)}\right)\right)\right\|_2^2 \\
\leq & \mathcal{L}^{{\text{meta} }}\left(\hat{\mathbf{w}}^{(1)}\right)\left(\Theta^{(1)}\right)-\mathcal{L}^{{\text{meta} }}\left(\hat{\mathbf{w}}^{(T+1)}\right)\left(\Theta^{(T+1)}\right)+\frac{M\left(\sigma^2+n^2\rho^2\right)}{2} \sum_{t=1}^T {(\lambda_1^{(t)})}^2\\
& \,+\rho \sum_{t=1}^T \lambda_1^{(t)} \left\|\mathcal{B}^{(t)}\right\|_2+\frac{M}{2} \sum_{t=1}^T {(\lambda_2^{(t)})}^2 \sigma^2 \\
\leq & \mathcal{L}^{{\text{meta} }}\left(\hat{\mathbf{w}}^{(1)}\right)\left(\Theta^{(1)}\right)+\frac{M\left(\sigma^2+n^2\rho^2\right)}{2} \sum_{t=1}^T {(\lambda_1^{(t)})}^2 +\rho \sum_{t=1}^T \lambda_1^{(t)} \left\|\mathcal{B}^{(t)}\right\|_2+\frac{M}{2} \sum_{t=1}^T {(\lambda_2^{(t)})}^2 \sigma^2.
\end{aligned}
\end{equation}
}}
Furthermore, we can deduce that
\begin{equation*}
\begin{aligned}
& \min _t \mathbb{E}\left[\left\|\nabla_{\Theta} \mathcal{L}^{{\text{meta} }}\left(\hat{\mathbf{w}}^{(t+1)}\left(\Theta^{(t)}\right)\right)\right\|_2^2\right] \\
\leq & \frac{\sum_{t=1}^T \left(\lambda_2^{(t)}-\frac{M {\lambda_2^{(t)}}^2}{2}\right) \mathbb{E}\left\|\mathcal{X}^{(t+1)}\right\|_2^2}
{\sum_{t=1}^T \left(\lambda_2^{(t)}-\frac{M {\lambda_2^{(t)}}^2}{2}\right)} \\
\leq & \frac{\mathcal{L}^{{\text{meta} }}\left(\hat{\mathbf{w}}^{(1)}\right)\left(\Theta^{(1)}\right)+\frac{M\left(\sigma^2+n^2\rho^2\right)}{2} \sum_{t=1}^T {(\lambda_1^{(t)})}^2+\rho \sum_{t=1}^T \lambda_1^{(t)} \left\|\mathcal{B}^{(t)}\right\|_2+\frac{M}{2} \sum_{t=1}^T {(\lambda_2^{(t)})}^2 \sigma^2 }{\sum_{t=1}^T \left(2 \lambda_2^{(t)}-{M {\lambda_2^{(t)}}^2}\right)} \\
\end{aligned}
\end{equation*}
\begin{equation}
\label{eq32}
\begin{aligned}
\leq & \frac{2\mathcal{L}^{{\text{meta} }}\left(\hat{\mathbf{w}}^{(1)}\right)\left(\Theta^{(1)}\right)+{M\left(\sigma^2+n^2\rho^2\right)} \sum_{t=1}^T {(\lambda_1^{(t)})}^2+\rho \sum_{t=1}^T \lambda_1^{(t)} \left\|\mathcal{B}^{(t)}\right\|_2+\frac{M}{2} \sigma^2 \sum_{t=1}^T {(\lambda_2^{(t)})}^2 }{\sum_{t=1}^T  \lambda_2^{(t)}} \\
\leq & \frac{2\mathcal{L}^{{\text{meta} }}\left(\hat{\mathbf{w}}^{(1)}\right)\left(\Theta^{(1)}\right)+{M\left(\sigma^2+n^2\rho^2\right)} \sum_{t=1}^T {(\lambda_1^{(t)})}^2+\rho \sum_{t=1}^T \lambda_1^{(t)} \left\|\mathcal{B}^{(t)}\right\|_2+\frac{M}{2} \sigma^2 \sum_{t=1}^T {(\lambda_2^{(t)})}^2}{L \lambda_2^{(T)}} \\
= & \frac{2\mathcal{L}^{{\text{meta} }}\left(\hat{\mathbf{w}}^{(1)}\right)\left(\Theta^{(1)}\right)+{M\left(\sigma^2+n^2\rho^2\right)} \sum_{t=1}^T {(\lambda_1^{(t)})}^2+\rho \sum_{t=1}^T \lambda_1^{(t)} \left\|\mathcal{B}^{(t)}\right\|_2+\frac{M}{2} \sigma^2 \sum_{t=1}^T {(\lambda_2^{(t)})}^2}{L} \\
& \, \times \max \left\{M, \frac{\sqrt{T}}{c}\right\} \\
= & \mathcal{O}\left(\frac{C}{\sqrt{T}}\right) .
\end{aligned}
\end{equation}
In Eq. (\ref{eq32}), the third inequality  holds since $\sum_{t=1}^T \left(2 \lambda_2^{(t)}-{M {\lambda_2^{(t)}}^2}\right)=\sum_{t=1}^T \lambda_2^{(t)}\left(2-M\lambda_2^{(t)}\right) \geq \sum_{t=1}^T \lambda_2^{(t)}$, and the fourth equality  holds since $\lim\limits_{T \rightarrow \infty} \sum_{t=1}^T {\lambda_1^{(t)}}^2<\infty, \lim\limits_{T \rightarrow \infty} \sum_{t=1}^T {\lambda_2^{(t)}}^2<\infty, \lim\limits_{T \rightarrow \infty} \sum_{t=1}^T \lambda_1^{(t)}\left\|\nabla \mathcal{L}^{t rain}\left(\mathbf{w}^{(t)} ; \Theta^{(t)}\right)\right\|_2<\infty$. Thus we can conclude that our algorithm can always achieve $\min\limits_{0 \leq t \leq T} \mathbb{E}\left[\left\|\nabla \mathcal{L}^{ {\text{meta} }}\left(\Theta^{(t)}\right)\right\|_2^2\right] \leq \mathcal{O}\left(\frac{C}{\sqrt{T}}\right)$ in $T$ steps, and this finishes the proof of Theorem \ref{th1}.
\end{proof}


\section{Proof for Theorem \ref{th2}}\label{7.4}

\begin{proof}
It is easy to conclude that $\lambda_1^{(t)}$ satisfy $\sum_{t=1}^{\infty} \lambda_1^{(t)}=\infty, \sum_{t=1}^{\infty} {(\lambda_1^{(t)})}^2<\infty$. 
Under the sampled mini-batch $\Psi_t$ from the training dataset, the update equation of $\mathbf{w}$ can be rewritten as:
$$
\mathbf{w}^{(t+1)}=\mathbf{w}^{(t)}-\left.\lambda_1^{(t)} \nabla \mathcal{L}^{t raim}\left(\mathbf{w}^{(t)} ; \Theta^{(t+1)}\right)\right|_{\Psi_t},
$$
where $\nabla \mathcal{L}^{ {\text{train} }}\left(\mathbf{w}^{(t)} ; \Theta^{(t)}\right)=$ $\left.\sum\limits_{i=1}^n \alpha_i^{(t)} L_{i,\text{train}}^1(\mathbf{w}) + \beta_i^{(t)} L_{i,\text{train}}^2(\mathbf{w}) \right|_{\mathbf{w}^{(t)}}$,
$\alpha_i^{(t)} =\mathcal{V}\left(L_{i,\text{train}}^1(\mathbf{w}^{(t)}),L_{i,\text{train}}^2(\mathbf{w}^{(t)}); \Theta\right)$. 
Since the mini-batch $\Psi_t$ is drawn uniformly at random, we can rewrite the update equation as:
$$
\mathbf{w}^{(t+1)}=\mathbf{w}^{(t)}-\lambda_1^{(t)}\left[\nabla \mathcal{L}^{t rain}\left(\mathbf{w}^{(t)} ; \Theta^{(t+1)}\right)+\psi^{(t)}\right],
$$
where $\psi^{(t)}=\left.\nabla \mathcal{L}^{t rain}\left(\mathbf{w}^{(t)} ; \Theta^{(t)}\right)\right|_{\Psi_t}-\nabla \mathcal{L}^{t rain}\left(\mathbf{w}^{(t)} ; \Theta^{(t)}\right)$. Note that $\psi^{(t)}$ are i.i.d. random variables with finite variance, since $\Psi_t$ are drawn i.i.d. with a finite number of samples, and thus $\mathbb{E}\left[\psi^{(t)}\right]=0, \mathbb{E}\left[\left\|\psi^{(t)}\right\|_2^2\right] \leq \sigma^2$.

The objective function $\mathcal{L}^{t rain}(\mathbf{w} ; \Theta)$ defined in Eq. (\ref{eq18}) can be easily proved to be Lipschitz-smooth with constant $M$, and have $\rho$-bounded gradient with respect to training set. Observe that
\begin{equation}
\label{eq34}
\begin{aligned}
& \mathcal{L}^{t rain}\left(\mathbf{w}^{(t+1)} ; \Theta^{(t+1)}\right)-\mathcal{L}^{t rain}\left(\mathbf{w}^{(t)} ; \Theta^{(t)}\right) \\
= & \left(\mathcal{L}^{t rain}\left(\mathbf{w}^{(t+1)} ; \Theta^{(t+1)}\right)-\mathcal{L}^{t rain}\left(\mathbf{w}^{(t+1)} ; \Theta^{(t)}\right)\right)\\
& +\left(\mathcal{L}^{t rain}\left(\mathbf{w}^{(t+1)} ; \Theta^{(t)}\right)-\mathcal{L}^{t rain}\left(\mathbf{w}^{(t)} ; \Theta^{(t)}\right)\right) .
\end{aligned}
\end{equation}

For the first  term, we have
\begin{equation}
\label{eq35}
\begin{aligned}
& \mathcal{L}^{\text{train}}\left(\mathbf{w}^{(t+1)} ; \Theta^{(t+1)}\right) - \mathcal{L}^{\text{train}}\left(\mathbf{w}^{(t+1)} ; \Theta^{(t)}\right) \\
= & \frac{1}{n} \sum_{i=1}^n \left( \mathcal{V}_i^{(t+1)}(\Theta^{(t+1)}) - \mathcal{V}_i^{(t+1)}(\Theta^{(t)}) \right) \times L_{j,\text{train}-}(\mathbf{w}^{(t+1)}) \\
\leq & \frac{1}{n} \sum_{i=1}^n \left\{ \left\langle \frac{\partial \mathcal{V}_i^{(t+1)}}{\partial \Theta} \bigg|_{\Theta^{(t)}}, \Theta^{(t+1)} - \Theta^{(t)} \right\rangle + \frac{\delta}{2} \left\| \Theta^{(t+1)} - \Theta^{(t)} \right\|_2^2 \right\}  \times L_{j,\text{train}-}(\mathbf{w}^{(t+1)}) \\
= & \frac{1}{n} \sum_{i=1}^n \left\{ \left\langle \frac{\partial \mathcal{V}_i^{(t+1)}}{\partial \Theta} \bigg|_{\Theta^{(t)}}, -\lambda_2^{(t)} \left[ \nabla \mathcal{L}^{\text{meta}}\left(\hat{\mathbf{w}}^{(t+1)}\left(\Theta^{(t)}\right)\right) + \xi^{(t)} \right] \right\rangle \right. \\
& \left. + \frac{\delta {(\lambda_2^{(t)})}^2}{2} \left( \left\|\nabla \mathcal{L}^{\text{meta}}\left(\hat{\mathbf{w}}^{(t+1)}\left(\Theta^{(t)}\right)\right)\right\|_2^2 + \left\|\xi^{(t)}\right\|_2^2 + 2 \left\langle \nabla \mathcal{L}^{\text{meta}}\left(\hat{\mathbf{w}}^{(t+1)}\left(\Theta^{(t)}\right)\right), \xi^{(t)} \right\rangle \right) \right\} \\
& \times L_{j,\text{train}-}(\mathbf{w}^{(t+1)}),
\end{aligned}
\end{equation}
where \(\mathcal{V}_i^{(t+1)}(\Theta)=\mathcal{V}\left(L_{i,\text{train}}^1(\mathbf{w}^{(t+1)}), L_{i,\text{train}}^2(\mathbf{w}^{(t+1)}); \Theta \right)\), and $L_{j,\text{train}-}(\mathbf{w}^{(t+1)}) = L_{j,\text{train}}^1(\mathbf{w}^{(t+1)}) - L_{j,\text{train}}^2(\mathbf{w}^{(t+1)}). $

For the second term, by Lipschitz smoothness of the training loss function $\mathcal{L}^{t rain}\left(\hat{\mathbf{w}}^{(t+1)}\left(\Theta^{(t+1)}\right)\right)$, we have
\begin{equation}
\label{eq36}
\begin{aligned}
& \mathcal{L}^{t rain}\left(\mathbf{w}^{(t+1)} ; \Theta^{(t)}\right)-\mathcal{L}^{t rain}\left(\mathbf{w}^{(t)} ; \Theta^{(t)}\right) \\
\leq & \left\langle\nabla \mathcal{L}^{t rain}\left(\mathbf{w}^{(t)} ; \Theta^{(t)}\right), \mathbf{w}^{(t+1)}-\mathbf{w}^{(t)}\right\rangle+\frac{M}{2}\left\|\mathbf{w}^{(t+1)}-\mathbf{w}^{(t)}\right\|_2^2 \\
= & \left\langle\nabla \mathcal{L}^{t rain}\left(\mathbf{w}^{(t)} ; \Theta^{(t)}\right),-\lambda_1^{(t)}\left[\nabla \mathcal{L}^{t rain}\left(\mathbf{w}^{(t)} ; \Theta^{(t)}\right)+\psi^{(t)}\right]\right\rangle\\
& +\frac{M {(\lambda_1^{(t)})}^2}{2}\left\|\nabla \mathcal{L}^{t rain}\left(\mathbf{w}^{(t)} ; \Theta^{(t)}\right)+\psi^{(t)}\right\|_2^2 \\
= & -\left(\lambda_1^{(t)}-\frac{M {(\lambda_1^{(t)})}^2}{2}\right)\left\|\nabla \mathcal{L}^{t rain}\left(\mathbf{w}^{(t)} ; \Theta^{(t)}\right)\right\|_2^2+\frac{M {(\lambda_1^{(t)})}^2}{2}\left\|\psi^{(t)}\right\|_2^2\\
& \, -\left(\lambda_1^{(t)}-M {(\lambda_1^{(t)})}^2\right)\left\langle\nabla \mathcal{L}^{t rain}\left(\mathbf{w}^{(t)} ; \Theta^{(t)}\right), \psi^{(t)}\right\rangle .
\end{aligned}
\end{equation}

Therefore, substituting Eq.(\ref{eq35}) and Eq.(\ref{eq36}) into Eq.(\ref{eq34}) yields
\begin{equation}
\label{eq37}
\begin{aligned}
& \mathcal{L}^{t rain}\left(\mathbf{w}^{(t+1)} ; \Theta^{(t+1)}\right)-\mathcal{L}^{t rain}\left(\mathbf{w}^{(t)} ; \Theta^{(t)}\right) \\
\leq & \frac{1}{n} \sum_{i=1}^n\left\{\left\langle\left.\frac{\partial \mathcal{V}_i^{(t+1)}(\Theta)}{\partial \Theta}\right|_{\Theta^{(t)}},-\lambda_2^{(t)}\left[\nabla \mathcal{L}^{{\text{meta} }}\left(\hat{\mathbf{w}}^{(t+1)}\left(\Theta^{(t)}\right)\right)+\xi^{(t)}\right]\right\rangle \right.\\
&\left.+\frac{\delta {(\lambda_2^{(t)})}^2}{2}\left(\left\|\nabla \mathcal{L}^{{\text{meta} }}\left(\hat{\mathbf{w}}^{(t+1)}\left(\Theta^{(t)}\right)\right)\right\|_2^2+\left\|\xi^{(t)}\right\|_2^2+2\left\langle\nabla \mathcal{L}^{{\text{meta} }}\left(\hat{\mathbf{w}}^{(t+1)}\left(\Theta^{(t)}\right)\right), \xi^{(t)}\right\rangle\right)\right\}\\
& \times L_{j,\text{train}-}(\mathbf{w}^{(t+1)})
-\left(\lambda_1^{(t)}-\frac{M {(\lambda_1^{(t)})}^2}{2}\right)\left\|\nabla \mathcal{L}^{t rain}\left(\mathbf{w}^{(t)} ; \Theta^{(t)}\right)\right\|_2^2+\frac{M {(\lambda_1^{(t)})}^2}{2}\left\|\psi^{(t)}\right\|_2^2\\
& -\left(\lambda_1^{(t)}-M {(\lambda_1^{(t)})}^2\right)\left\langle\nabla \mathcal{L}^{t rain}\left(\mathbf{w}^{(t)} ; \Theta^{(t)}\right), \psi^{(t)}\right\rangle .
\end{aligned}
\end{equation}

Taking expectation of the both sides of Eq.(\ref{eq37}) and based on $\mathbb{E}\left[\xi^{(t)}\right]=0, \mathbb{E}\left[\psi^{(t)}\right]=0$, we have

\begin{equation}
\label{eq38}
\begin{aligned}
& \mathbb{E}\left[\mathcal{L}^{ {\text{train} }}\left(\mathbf{w}^{(t+1)} ; \Theta^{(t+1)}\right)\right]-\mathbb{E}\left[\mathcal{L}^{ {\text{train} }}\left(\mathbf{w}^{(t)} ; \Theta^{(t)}\right)\right] \\
\leq  & \mathbb{E}\frac{1}{n} \sum_{i=1}^n\left\{\left\langle\left.\frac{\partial \mathcal{V}_i^{(t+1)}(\Theta)}{\partial \Theta}\right|_{\Theta^{(t)}},-\lambda_2^{(t)}\nabla \mathcal{L}^{{\text{meta} }}\left(\hat{\mathbf{w}}^{(t+1)}\left(\Theta^{(t)}\right)\right)\right\rangle \right.\\
& \left.+\frac{\delta {(\lambda_2^{(t)})}^2}{2}\left(\left\|\nabla \mathcal{L}^{{\text{meta} }}\left(\hat{\mathbf{w}}^{(t+1)}\left(\Theta^{(t)}\right)\right)\right\|_2^2+\left\|\xi^{(t)}\right\|_2^2\right)\right\}
\times  L_{j,\text{train}-}(\mathbf{w}^{(t+1)})\\
& -\left(\lambda_1^{(1)}-\frac{M {(\lambda_1^{(t)})}^2}{2}\right)\mathbb{E}\left[\left\|\nabla \mathcal{L}^{t rain}\left(\mathbf{w}^{(t)} ; \Theta^{(t)}\right)\right\|_2^2\right]+\frac{M {(\lambda_1^{(t)})}^2}{2}\mathbb{E}\left[\left\|\psi^{(t)}\right\|_2^2\right].
\end{aligned}
\end{equation}
According to $\frac{1}{n} \sum_{i=1}^n\left\|L_{i,\text{train}}^1\left(\mathbf{w}^{(t)}\right)\right\| \leq K$, $\frac{1}{n} \sum_{i=1}^n\left\|L_{i,\text{train}}^2\left(\mathbf{w}^{(t)}\right)\right\| \leq K$ and triangle inequality, we can get $\frac{1}{n} \sum_{i=1}^n\left\|L_{i,\text{train}}^1\left(\mathbf{w}^{(t)}\right)-L_{i,\text{train}}^2\left(\mathbf{w}^{(t)}\right)\right\| \leq 2K$. Since $\sum_{t=0}^{\infty} {(\lambda_1^{(t)})}^2<\infty, \sum_{t=0}^{\infty} {(\lambda_2^{(t)})}^2<\infty, \sum_{t=1}^T {(\lambda_2^{(t)})}\left\|\nabla \mathcal{L}^{{\text{meta} }}\left(\hat{\mathbf{w}}^{(t+1)}\left(\Theta^{(t)}\right)\right)\right\|<\infty$, and 
Summing up the above inequalities over $t=1, \ldots, T$ in both sides and rearranging the terms, we obtain

$$
\begin{aligned}
& \sum_{t=1}^T\left(\lambda_1^{(t)}-\frac{M {(\lambda_1^{(t)})}^2}{2}\right)\mathbb{E}\left\|\nabla \mathcal{L}^{t rain}\left(\mathbf{w}^{(t)} ; \Theta^{(t)}\right)\right\|_2^2 \\
\leq & \sum_{t=1}^T \mathbb{E}\frac{1}{n} \sum_{i=1}^n\left\{\left\langle\left.\frac{\partial \mathcal{V}_i^{(t+1)}(\Theta)}{\partial \Theta}\right|_{\Theta^{(t)}},-\lambda_2^{(t)}\nabla \mathcal{L}^{{\text{meta} }}\left(\hat{\mathbf{w}}^{(t+1)}\left(\Theta^{(t)}\right)\right)\right\rangle \right.\\
&\left.
+\frac{\delta {(\lambda_2^{(t)})}^2}{2}\left(\left\|\nabla \mathcal{L}^{ {\text{meta} }}\left(\hat{\mathbf{w}}^{(t+1)}\left(\Theta^{(t)}\right)\right)\right\|_2^2+\left\|\xi^{(t)}\right\|_2^2\right)\right\}
 \times L_{j,\text{train}-}(\mathbf{w}^{(t+1)})\\
&+\sum_{t=1}^T\frac{M {(\lambda_1^{(1)})}^2}{2}\mathbb{E}\left[\left\|\psi^{(t)}\right\|_2^2\right]+\mathbb{E}\left[\mathcal{L}^{ {\text{train} }}\left(\mathbf{w}^{(1)} ; \Theta^{(1)}\right)\right]-\mathbb{E}\left[\mathcal{L}^{ {\text{train} }}\left(\mathbf{w}^{(T+1)} ; \Theta^{(T+1)}\right)\right]\\
\end{aligned}
$$
$$
\begin{aligned}
\leq & \sum_{t=1}^T \lambda_2^{(t)} \mathbb{E} \frac{1}{n} \sum_{i=1}^n\left\|L_{j,\text{train}-}(\mathbf{w}^{(t+1)})\right\|\left\|\left.\frac{\partial \mathcal{V}_i^{(t+1)}(\Theta)}{\partial \Theta}\right|_{\Theta^{(t)}}\right\|\left\|\nabla \mathcal{L}^{{\text{meta} }}\left(\hat{\mathbf{w}}^{(t+1)}\left(\Theta^{(t)}\right)\right)\right\|\\
&+\sum_{t=1}^T \frac{M {(\lambda_1^{(1)})}^2}{2} \mathbb{E}\left[\left\|\psi^{(t)}\right\|_2^2\right]+\mathbb{E}\left[\mathcal{L}^{{\text{train} }}\left(\mathbf{w}^{(1)} ; \Theta^{(1)}\right)\right]-\mathbb{E}\left[\mathcal{L}^{t rain}\left(\mathbf{w}^{(T+1)} ; \Theta^{(T+1)}\right)\right]\\
&+\sum_{t=1}^T \frac{\delta {(\lambda_2^{(t)})}^2}{2}\left\{\frac{1}{n} \sum_{i=1}^n\left\|L_{j,\text{train}-}(\mathbf{w}^{(t+1)})\right\|\left(\mathbb{E}\left\|\nabla \mathcal{L}^{{\text{meta} }}\left(\hat{\mathbf{w}}^{(t+1)}\left(\Theta^{(t)}\right)\right)\right\|_2^2+\mathbb{E}\left\|\xi^{(t)}\right\|_2^2\right)\right\} \\
\leq & 2\delta K \sum_{t=1}^T \lambda_2^{(t)}\left\|\nabla \mathcal{L}^{ {\text{meta} }}\left(\hat{\mathbf{w}}^{(t+1)}\left(\Theta^{(t)}\right)\right)\right\|+\sum_{t=1}^T \frac{M {(\lambda_1^{(t)})}^2}{2} \sigma^2\\
& +\mathbb{E}\left[\mathcal{L}^{t rain}\left(\mathbf{w}^{(1)} ; \Theta^{(1)}\right)\right]+\sum_{t=1}^T \frac{\delta {(\lambda_2^{(t)})}^2}{2}\left\{K\left(\rho^2+\sigma^2\right)\right\}<\infty .
\end{aligned}
$$

Since $\sum_{t=1}^T\left(2 \lambda_1^{(t)}-M {(\lambda_1^{(t)})}^2\right)=\sum_{t=1}^T \lambda_1^{(t)}\left(2-M\lambda_1^{(t)}\right) \geq \sum_{t=1}^T \lambda_1^{(t)}$, we have
$$
\begin{aligned}
& \min _t \mathbb{E}\left[\left\|\nabla \mathcal{L}^{t rain}\left(\mathbf{w}^{(t)} ; \Theta^{(t)}\right)\right\|_2^2\right] \\
\leq & \frac{2\delta K \sum_{t=1}^T \lambda_2^{(t)}\left\|\nabla \mathcal{L}^{ {\text{meta} }}\left(\hat{\mathbf{w}}^{(t+1)}\left(\Theta^{(t)}\right)\right)\right\|+\sum_{t=1}^T \frac{M {(\lambda_1^{(t)})}^2}{2} \sigma^2+\mathcal{D}}{\sum_{t=1}^T\left(\lambda_1^{(t)}-\frac{M {(\lambda_1^{(t)})}^2}{2}\right)} \\
\leq & \frac{2\delta K \sum_{t=1}^T \lambda_2^{(t)}\left\|\nabla \mathcal{L}^{ {\text{meta} }}\left(\hat{\mathbf{w}}^{(t+1)}\left(\Theta^{(t)}\right)\right)\right\|+\sum_{t=1}^T \frac{M {(\lambda_1^{(t)})}^2}{2} \sigma^2+\mathcal{D}}{\sum_{t=1}^T\lambda_1^{(t)}} \\
\leq & \frac{2\delta K \sum_{t=1}^T \lambda_2^{(t)}\left\|\nabla \mathcal{L}^{ {\text{meta} }}\left(\hat{\mathbf{w}}^{(t+1)}\left(\Theta^{(t)}\right)\right)\right\|+\sum_{t=1}^T \frac{M {(\lambda_1^{(t)})}^2}{2} \sigma^2+\mathcal{D}}{T\lambda_1^{(t)}} \\
\leq & \frac{2\delta K \sum_{t=1}^T \lambda_2^{(t)}\left\|\nabla \mathcal{L}^{ {\text{meta} }}\left(\hat{\mathbf{w}}^{(t+1)}\left(\Theta^{(t)}\right)\right)\right\|+\sum_{t=1}^T \frac{M {(\lambda_1^{(t)})}^2}{2} \sigma^2+\mathcal{D}}{T} \max \left\{M, \frac{\sqrt{T}}{c}\right\} \\ =&\mathcal{O}\left(\frac{C}{\sqrt{T}}\right),  \\
&
\end{aligned}
$$
where $ \mathcal{D}= \mathbb{E}\left[\mathcal{L}^{t rain}\left(\mathbf{w}^{(1)} ; \Theta^{(1)}\right)\right]+\sum_{t=1}^T \frac{\delta {(\lambda_2^{(t)})}^2}{2}\left\{K\left(\rho^2+\sigma^2\right)\right\} $.
Thus,  we can conclude that the algorithm can always achieve $$\min\limits_{0 \leq t \leq T} \mathbb{E}\left[\left\|\nabla \mathcal{L}^{t rain}\left(\mathbf{w}^{(t)} ; \Theta^{(t)}\right)\right\|_2^2\right] \leq \mathcal{O}\left(\frac{C}{\sqrt{T}}\right)$$ in $T$ steps, and this completes the proof of Theorem \ref{th2}.
\end{proof}

\end{document}